%% file: paper.tex
\documentclass{article}
\usepackage{arxiv}

\usepackage{amsthm,amssymb}
\usepackage{mathtools,thmtools}
\usepackage{bm,esint}
\usepackage[dvipsnames]{xcolor}
\usepackage{float,graphicx,subcaption}
\usepackage{cancel}
\usepackage[shortlabels]{enumitem}
\usepackage{mathrsfs}  
\usepackage{bbm}
\usepackage{euscript}
\usepackage{upgreek}
\usepackage[linesnumbered]{algorithm2e}
\usepackage{bbm}
\usepackage{empheq}
\usepackage{hyperref}
\hypersetup{
    colorlinks=true,
    linkcolor=blue,
    filecolor=magenta,      
    urlcolor=cyan,
    pdftitle={Overleaf Example},
    pdfpagemode=FullScreen,
    }
\usepackage{bbm}

\usepackage[color=blue!20!white]{todonotes}

\input{custom_defs}
\input{custom_thms}

\title{Data Complexity Estimates \\ for Operator Learning}
\renewcommand{\shorttitle}{Data Complexity Estimates for Operator Learning} 
\author{Nikola B. Kovachki, Samuel Lanthaler, Hrushikesh Mhaskar}
\date{\today}

\ifpdf
\makeatletter
\hypersetup{
  pdftitle={\@title},
  pdfauthor={\@author}
}
\makeatother
\fi

\usepackage{graphicx}

\begin{document}

\maketitle

\begin{abstract}
Operator learning has emerged as a new paradigm for the data-driven approximation of nonlinear operators. Despite its empirical success, the theoretical underpinnings governing the conditions for efficient operator learning remain incomplete. The present work develops theory to study the data complexity of operator learning, complementing existing research on the parametric complexity. We investigate the fundamental question: How many input/output samples are needed in operator learning to achieve a desired accuracy $\epsilon$? This question is addressed from the point of view of $n$-widths, and this work makes two key contributions. The first contribution is to derive lower bounds on $n$-widths for general classes of Lipschitz and Fr\'echet differentiable operators. These bounds rigorously demonstrate a ``curse of data-complexity'', revealing that learning on such general classes requires a sample size exponential in the inverse of the desired accuracy $\epsilon$. The second contribution of this work is to show that ``parametric efficiency'' implies ``data efficiency''; using the Fourier neural operator (FNO) as a case study, we show rigorously that on a narrower class of operators, efficiently approximated by FNO in terms of the number of tunable parameters, efficient operator learning is attainable in data complexity as well. Specifically, we show that if only an algebraically increasing number of tunable parameters is needed to reach a desired approximation accuracy, then an algebraically bounded number of data samples is also sufficient to achieve the same accuracy.
\end{abstract}

\section{Introduction}

In recent years, operator learning has emerged as a new paradigm for the data-driven approximation of operators arising in engineering and the physical sciences. Popular operator learning frameworks build on deep neural networks, suitably generalized to infinite dimensions. They approximate nonlinear operators, mapping input functions to output functions. Such operators abound in scientific computing and are often associated with the solution operator of an underlying partial differential equation (PDE), e.g. mapping initial conditions, boundary data, or coefficient fields to the solution of the PDE. In these applications, operator learning frameworks have shown great potential as surrogate models, achieving significant gains in computational efficiency over traditional numerical methods. Since they are data-driven, these methods can also be used to approximate operators for which a mathematical description in terms of a PDE is incomplete or not available.

While there is an increasing body of empirical evidence demonstrating the practical efficacy of operator learning frameworks, our theoretical understanding of these methods is still far from complete. There are several papers studying operator learning from an approximation theoretic perspective, aiming to bound the total number of parameters (model size) required to approximate a given operator, or general classes of operators, to a desired approximation accuracy. For a recent review, we refer to \cite{kls2024hna}. We highlight work on holomorphic operators, which represent a general class of operators for which \emph{efficient} quantitative error and complexity estimates can be derived. Approximation theoretic results for holomorphic operators build on the influential work by Cohen, DeVore and Schwab \cite{cohen2010convergence,cohen2011analytic}, including further developments in \cite{schwab2019deep,opschoor2022exponential,sz2023,herrmann2022neural,marcati2023exponential,lanthaler2021error,sz2023}. Under suitable conditions, these results allow one to prove that ReLU deep operator networks (DeepONet) approximating holomorphic operators can achieve convergence rates that decay algebraically in the number of tunable parameters of the architecture. Algebraic convergence rates have also been obtained for operator Barron spaces, which are discussed in the recent paper \cite{korolev2021two}, and a related context of operator reproducing kernel Hilbert spaces (RKHS) \cite{nelsen2021random,lanthalernelsen2023error}.

The previously mentioned results are restricted to either holomorphic operators or operators belonging to a Barron/RKHS space. Another general and natural class of non-linear operators are general Lipschitz continuous operators. Upper error and complexity bounds for the approximation of such operators have been considered e.g. in \cite{pca,liu2022deep,franco2023approximation,galimberti2023designing,schwab2023deep}. Recently, lower complexity bounds, identifying limitations of operator learning have also been obtained \cite{lanthaler2023curse,lanthaler2023operator}. In particular, it has been shown that operator learning on general classes of Lipschitz continuous and Fr\'echet differentiable operators generally suffers from a ``curse of parametric complexity''; These results show that there exist Fr\'echet differentiable operators of any given smoothness $k$, for which the approximation to a desired uniform accuracy $\epsilon$ requires models with a number of parameters that is exponential, $\gtrsim \exp(c\epsilon^{-\gamma})$ in the inverse of the accuracy. This exponential ``curse'' appears as a scaling limit of the well-known curse of dimensionality encountered in high-dimensional approximation problems and applies to popular operator learning frameworks such as DeepONets, the Fourier neural operator, or PCA-Net \cite{lu2019deeponet, li2020fourier, pca}. In this context, we mention relevant work \cite{schwab2023deep}, where ``hyperexpressive'' (non-standard) neural networks are shown to overcome this curse when the complexity is measured by the number of tunable (real-valued) parameters. However, even with these non-standard architectures, the curse re-appears when accounting for the number of bits required to represent the real parameters to sufficient precision on computational hardware \cite{lanthaler2024lipschitzoperators}.

The complementary question of the data-complexity of operator learning has received considerably less attention in the literature. The sample complexity of operator learning for holomorphic operators has been studied in \cite{bachmayr2017kolmogorov,adcock2023a,adcock2023b}. It has been shown that, for this class of operators, approximation to accuracy $\epsilon$ can be achieved with a number of samples that only scales algebraically in $\epsilon^{-1}$. 

For general classes of Lipschitz continuous operators, we mention the early work by \cite{mhaskar1997neural}. This work derives matching upper and lower bounds on the continuous non-linear $n$-widths of spaces of non-linear Lipschitz functionals defined on $L^2$-spaces, showing that such $n$-widths only decay logarithmically, as $\log(n)^{-\lambda}$ for fixed $\lambda$ depending on the specific assumptions.

Another paper of note is \cite{liu2022deep}. In this work, non-asymptotic upper bounds for the generalization error of empirical risk minimizers on suitable classes of operator networks are derived, achieving logarithmic decay rates that qualitatively match the $n$-width bounds of \cite{mhaskar1997neural}. In a related direction, the authors of \cite{kim2022bounding} provide estimates on the Rademacher complexity of FNO, allowing generalization error estimates to be derived. We finally mention the recent work \cite{mukherjee2023size}, where a connection is made between the number of available samples and the required size of the DeepONet reconstruction dimension.

The general curse of parametric complexity identified in \cite{lanthaler2023curse,lanthaler2024lipschitzoperators}, and the $n$-width estimates derived in the specific setting of \cite{mhaskar1997neural} suggest that operator learning on general classes of operators might suffer from an analogous ``curse of sample complexity'', potentially requiring an exponential number of samples to achieve a desired approximation accuracy. A general theory which would imply such a curse of data-complexity has so far been outstanding. Furthermore, the empirically observed efficiency of operator learning for concrete problems raises the question how such a theoretical exponential lower bound on the sample complexity of operator learning can be reconciled with practical observations.

The results of the present paper shed new light on these fundamental questions. We develop lower and upper bounds on the sample complexity of operator learning. Our main contributions are outlined as follows.

\paragraph{Lower bounds.} We prove that operator learning on general classes of operators, such as Lipschitz operators or continuously differentiable operators possessing $k$ bounded Fr\'echet derivatives, suffers from a curse of sample complexity; in the following prototypical settings we prove that approximation of such operators requires an exponential number of samples in a minimax sense that is independent of the reconstruction algorithm. We demonstrate such bounds for:
\begin{enumerate}
\item uniform approximation of operators over common compact sets of input functions, defined by a smoothness constraint,
\item approximation in a Bochner $L^p$-norm with respect to input functions drawn from a Gaussian measure with algebraically decaying covariance spectrum.
\end{enumerate}

\paragraph{Upper bounds.} Taking the Fourier neural operator (FNO) as a case study, we prove that on a narrower class of operators, efficiently approximated by FNO in terms of the required model size, the data complexity of operator learning is similarly benign; specifically, we introduce suitable spaces of operators $\cA^\gamma$ inspired by related work in finite dimensions \cite{grohs2023proof,berner2022learning,abdeljawad2023sampling}, and show that a number of samples that grows at most algebraically in the inverse of the desired approximation error $\epsilon^{-1}$ suffices for approximation of operators in $\cA^\gamma$. These results hold in the root mean square sense with respect to a probability measure $\mu$ of input functions.

\subsection{Outline}

In section \ref{sec:lower}, we provide a detailed overview of our lower data complexity bounds, including a sketch of their proofs. Complete details of the proofs of these lower data complexity bounds can be found in Appendix \ref{sec:bds_finite} for the finite-dimensional setting and Appendix \ref{sec:lower_bounds_infd} in the infinite-dimensional setting. A summary of our upper bounds for FNO are presented in section \ref{sec:upper}, with further details in Appendix \ref{sec:data-efficient}. Conclusions are presented in section \ref{sec:conclusion}. A short summary of notation followed in this paper, including a review of the notion of Fr\'echet differentiability studied in the first part, is included in Appendix \ref{sec:notation}.

\section{Lower data complexity bounds}
\label{sec:lower}

In the present section, we provide an overview of our results on lower data complexity bounds. To promote readability, we will only outline the proofs of these results, and refer the reader to the appendices for detailed derivations. Before stating our main results, we recall two notions of $n$-widths, which provide a measure of the complexity of a space of functions/operators, and on which our theoretical lower bounds are based. 

\subsection{Two notions of $n$-widths}
We recall two different notions of $n$-widths which will be the primary objects of study throughout the first part of present work. Let \((X,\|\cdot\|_X)\) be a normed vector space and let \(K \subset X\)
be a subset. Define the \textit{encoder} $\Enc_n : X \to \R^n$ to be any mapping such that $\Enc_n |_K$ is continuous in the topology of $X$. Let the \textit{decoder} $\D_n : \R^n \to X$ be any mapping. Following \cite{devore1989optimal}, we define the \textit{continuous, nonlinear \(n\)-width} of \(K\) in \(X\) by
\[d_n(K)_X \coloneqq\inf_{\Enc_n, \D_n} \sup_{f \in K} \|f - (\D_n \circ \Enc_n)(f)\|_X.\]

Suppose now that $X$ is a real-valued, Banach function space, in the sense of \cite[Definition 1.3]{bennett1988interpolation}, over a topological measure space $(\Omega, \mu)$. Suppose that there exists a unique, linear injection $\iota : K \to C (\Omega)$ such that $f = \iota f$ $\mu$-almost everywhere. Let $Z_n = \{z_1, \dots, z_n\} \subset \Omega$ be any set of points. We may then define, for any $f \in K$, the functional $\delta_{Z_n} : K \to \R^n$ by $\delta_{Z_n} (f) = \big( (\iota f) (z_1), \dots (\iota f) (z_n) \big )$. We define the \textit{sampling, nonlinear $n$-width} of $K$ in $X$ by
\[s_n(K)_X \coloneqq \inf_{Z_n, \cD_n} \sup_{f \in K} \|f - (\D_n \circ \delta_{Z_n})(f)\|_X.\]
Note that, in general, $\delta_{Z_n}$ will \textit{not} be continuous in the topology of $X$. A special case when continuity holds is $X = C(\Omega)$. Another such case is when $Y \hookrightarrow X$ is another Banach space which compactly embeds into $X$ and for which $\delta_{Z_n} : Y \to \R^n$ is continuous. Then for any bounded set $K \subset Y$, Lemma~\ref{lemma:cont_linear_map} in Appendix~\ref{sec:cont_lin_func} shows that the map $\delta_{Z_n}|_K$ is continuous in $X$. In both of these cases, we obtain
\[s_n(K)_{X} \geq d_n(K)_{X}.\]
A particularly important case which we will often consider is $X = L^p_\mu (\Omega)$ the Lebesgue space and $Y = W^{k,q}_\mu (\Omega)$ the Sobolev space with $K = \mathcal{U} (Y)$ the unit ball of $Y$. In this setting, whenever $Y$ embeds continuously in $C(\Omega)$ and compactly into $X$, the above lower bound holds.

\subsection{Bounds for function spaces}

In this section, we consider Banach function spaces defined over the measure space $\Omega = [0,1]^d$ equipped with the Lebesgue measure or $\Omega = \R^d$ equipped with a Gaussian measure. The finite-dimensional results of this section prepare our later extension to the infinite-dimensional context of operator learning considered in section \ref{sec:oplearn}.
Here, we establish lower bounds for the continuous, non-linear $n$-width $d_n$ and the sampling $n$-width $s_n$ of the 
unit ball of $C^k(\Omega)$ or $W^{k,q}(\Omega)$ in $L^p (\Omega)$ and of the unit ball of $C^k(\Omega)$ in $C(\Omega)$.
For $d_n$, with $\Omega = [0,1]^d$, such results have first been derived in \cite{devore1989optimal}; for our application to operator learning, it will be crucial to carefully track and control the dependency of these estimates on the dimension $d$. Direct application of the proof techniques in \cite{devore1989optimal}
appears to lead to an exponential dependence of these constants on $d$. This is too pessimistic for our purposes, where, informally speaking, we consider the limit $d \to \infty$.
We therefore present modified proofs for the results of \cite{devore1989optimal} which give algebraic dependence on $d$.

\subsubsection{Lower $n$-width bounds for $\Omega = [0,1]^d$}
\label{subsubsec:lower_compact}

The following theorem summarizes our sampling width estimate for the unit balls in the function spaces $W^{k,q}(\Omega)$ and $C^k(\Omega)$, with approximation error measured in the $L^p$-norm. We note that the sampling width is defined via point-evaluation; to apply this notion to spaces $W^{k,q}(\Omega)$, we impose that $k>d/q$ if $q > 1$ or $k \geq d$ if $q=1$, which ensures point-evaluation is well-defined by the Sobolev embedding $W^{k,q}(\Omega) \embeds C(\Omega)$ \cite{adams2003sobolev}. Note that here and throughout the rest of this work, we denote by $\cU (E)$, the unit ball of any normed space $E$; see Appendix \ref{sec:notation} for further details on notation.

\begin{theorem}
    \label{thm:sampling_width}
    Let $\Omega = [0,1]^d$. Let $K = \cU \big (W^{k,q} (\Omega) \big )$ for any $1 \leq q \leq \infty$ with $k > d/q$ if $q>1$ or $k\ge d$ if $q=1$, or let $K = \cU \big ( C^k (\Omega) \big )$ with $q = \infty$ and $k \in \N$. Let $1 \leq p \leq \infty$ then there exists a constant $Q = Q(k,q) > 0$ such that, for any $n \in \N$,
    \[s_n (K)_{L^p} \geq Q \left ( 1 + \frac{d}{kp} \right )^{-k} d^{-k/q} n^{-k/d}.\]
\end{theorem}

The following theorem summarizes the corresponding non-linear $n$-width estimates. These hold only under the restriction $p \ge q$, but are true for a more general class of continuous encoders than the encoding by sample evaluation considered in the previous Theorem \ref{thm:sampling_width}.

\begin{theorem}
    \label{thm:continous_width}
    Let $K = \cU \big (W^{k,q} (\Omega) \big )$ for any $1 \leq q \leq \infty$ or $K = \cU \big ( C^k (\Omega) \big )$ with $q = \infty$. Let $p \geq q$ then there exists a constant $Q = Q(k,q) > 0$ such that, for any $n \in \N$,
    \[d_n (K)_{L^p} \geq Q \left ( 1 + \frac{d}{kp} \right )^{-k} d^{-k/q} n^{-k/d}.\]
\end{theorem}

\begin{remark}
    In Theorems~\ref{thm:sampling_width} and \ref{thm:sampling_width}, the constant is independent of the dimension $d$ only in the case $p=q=\infty$.
    Furthermore, by considering this special case in each proof, we verify that the same lower bounds holds for $\cU \big( C^k (\Omega) \big )$
    in $C(\Omega)$.
\end{remark}

\begin{remark}
It has been claimed, for example in \cite{DeVore1993WaveletCA}, that $d_n$ has been asymptotically determined for all $p,q$ pairs, referring to \cite{devore1989optimal}. We could not verify the required estimates in the case $q > p$ by using the proof techniques from \cite{devore1989optimal}. In fact, this problem has only recently been settled by the contemporaneous work \cite{siegel2024sharp}. However, the proof of \cite{siegel2024sharp} yields a constant with an exponential dependence on the dimension $d$. Since the case $q>p$ (the specific setting where $q=\infty$, $p<\infty$ to be precise) is of particular relevance to operator learning, we circumvent this issue by instead establishing lower bounds on the sampling widths $s_n$ in Theorem \ref{thm:sampling_width}. Our estimate is valid for any $p,q$ pair, and the constant has an algebraic dependence on the dimension $d$.
\end{remark}

\subsubsection{Overview of the proofs of Theorems \ref{thm:sampling_width} and \ref{thm:continous_width}}

The proofs of Theorems \ref{thm:sampling_width} and \ref{thm:continous_width} follows along similar lines. The first step, detailed in Appendix \ref{subsec:widths_cube}, is to partition the unit cube $[0,1]^d = \bigcup_{j=1}^{n} Q_j$ into $n = m^d$ subcubes of side-length $1/m$, for fixed $m\in \N$. We next construct a suitable family of $n$ bump functions $\phi_{\gamma,j}$, $j=1,\dots, m^d$ supported in distinct cubes $Q_j$ of the partition. This approach is mostly analogous to \cite{devore1989optimal}, and the lower $n^{-k/d}$ bound follows by considering functions of the form 
\[
f(x) = J\sum_{j=1}^n \alpha_j \phi_{\gamma,j}(x),
\]
for carefully chosen scaling $J>0$, $|\alpha_j|\le 1$ and $n = m^d \in \N$.

Our main technical improvement over \cite{devore1989optimal} is the introduction of a free parameter $\gamma \in (0,1)$. This parameter controls the extent of the set $\{\phi_{\gamma,j} = 1\} \subset Q_j$, with $\phi_{\gamma,j} \to 1_{Q_j}$ as $\gamma \to 1$. The choice of $\gamma = \gamma(d)$ can then be optimized to control the ratio of $\Vert \phi_{\gamma,j} \Vert_{L^p} \sim \gamma^{d/p}$ to $\Vert \phi_{\gamma,j} \Vert_{C^k} \sim (1-\gamma)^{-k}$, resulting in the following $d$-dependency of the constant 
\[
C = Q\left ( 1 + \frac{d}{kp} \right )^{-k} d^{-k/q}
\]
which holds both in Theorems \ref{thm:sampling_width} and \ref{thm:continous_width}. We note that a fixed choice, such as $\gamma = 1/2$ as in \cite{devore1989optimal}, would give the same asymptotics in $n$, but would lead to a constant $C$ with exponential dependency on $d$. We refer to Appendix \ref{subsec:widths_cube} for the details of the required argument. 

\begin{remark}
    The argument used to prove Theorem~\ref{thm:continous_width} does not achieve the same lower bound for the case $q > p$.
    Tracing through the proof given in the appendix, we note that \eqref{eq:xn_lp_lowerbound} leads to 
    \[\sum_{j=1}^n |\alpha_j|^q \leq \left ( \sum_{j=1}^n |\alpha_j|^p \right )^{q/p} \leq n^{q/p} \gamma^{dq/p} \|f\|^q_p.\]
    Then \eqref{eq:xn_wkq_upperbound} implies
    \[\|f\|_{W^{k,q}} \lesssim n^{(k/d) + (q-p)/p}\]
    and therefore
    \[d_n (K)_{L^p} \gtrsim n^{-(k/d) - (q-p)/p}\]
    and the expression on the right-hand side appears to overcome the curse of dimensionality, decaying at least as fast as $n^{-(q-p)/p}$ even in the limit $d\to \infty$. 
    This inadequacy stems from using the Bernstein width to lower bound $d_n$ as proposed in \cite{devore1989optimal}. We refer to the recent work \cite{siegel2024sharp} for a new proof technique which overcomes this challenge.
\end{remark}

\subsubsection{Lower $n$-width bounds for Gaussian measures on $\Omega = \R^d$}
\label{subsec:lowerbounds_gaussian}

In this section, we consider $\Omega = \R^d$ but measure the approximation error in weighted $L^p_{\rho_d}$-norms, where $\rho_d$ denotes the standard Gaussian density on $\R^d$. The following theorem provides estimates on the sampling widths.

\begin{theorem}
    \label{thm:sampling_width_gaussian}
    Let $K = \cU \big (W^{k,q}_{\rho_d} (\R^d) \big ) \cap C(\R^d)$ for any $1 \leq q \leq \infty$ or $K = \cU \big ( C^k (\R^d) \big )$ with $q = \infty$. Let $1 \leq p \leq \infty$ then there exists a constant $Q = Q(k,q) > 0$ such that, for any $n \in \N$,
    \[s_n (K)_{L^p_{\rho_d}} \geq Q \left ( 1 + \frac{d}{kp} \right )^{-k} d^{-k/q} n^{-k/d}.\]
\end{theorem}

The following theorem summarizes estimates on the continuous non-linear $n$-widths in the weighted $L^p_{\rho_d}$ setting.

\begin{theorem}
    \label{thm:continous_width_gaussian}
    Let $K = \cU \big (W^{k,q}_{\rho_d} (\R^d) \big )$ for any $1 \leq q \leq \infty$ or $K = \cU \big ( C^k (\R^d) \big )$ with $q = \infty$. Let $p \geq q$ then there exists a constant $Q = Q(k,q) > 0$ such that, for any $n \in \N$,
    \[d_n (K)_{L^p_{\rho_d}} \geq Q \left ( 1 + \frac{d}{kp} \right )^{-k} d^{-k/q} n^{-k/d}.\]
\end{theorem}

\subsubsection{Overview of the proof of Theorems \ref{thm:sampling_width_gaussian} and \ref{thm:continous_width_gaussian}}

Detailed proofs of Theorems \ref{thm:sampling_width_gaussian} and \ref{thm:continous_width_gaussian} are given in Appendix \ref{subsec:width_gaussian}. The main difference with the case of $\Omega = [0,1]^d$ is that  the partition has to be modified so that the union of the support of all functions becomes $\R^d$ instead of $[0,1]^d$. We achieve this by defining an invertible transport map between the standard Gaussian measure on $\R^d$ and the uniform measure on $[0,1]^d$. In particular, we define $\xi : \mathbb{R} \to (0,1)$ by
\[\xi(x) = \int_{-\infty}^x \rho_1 (y) \: \mathsf{d}y, \]
to be the cumulative distribution function of $\rho_1$. We use $x \mapsto \xi(x)$, applied to each variable $x_1,\dots, x_d$, to define a measure-theoretic isomorphism between $L^p([0,1]^d)$ and $L^p_{\rho_d}(\R^d)$. After this transformation, the ideas mirror the estimates for $\Omega = [0,1]^d$ and are again based on a partition of the domain. The main difference in the present case are some additional technical complications arising due to the fact that derivatives now have to be applied to a composition of the bump functions $\phi_{\gamma,j} \circ \xi$ introduced in the previous section, and relevant change-of-variables mapping $\xi$. Importantly, the transformation $y \mapsto \xi^{-1}(y)$ ``stretches'' the domain, implying that gradients are generally reduced when passing from $\phi_{\gamma,j}$ to $\phi_{\gamma,j} \circ \xi$. This fact is the basis of our argument; we refer the reader to Appendix \ref{subsec:width_gaussian} for the technical details.

\subsection{Bounds for spaces of operators}
\label{sec:oplearn}

We will now consider the case when the goal is to approximate operators $\cG$ belonging to some class of operators $\U$ and inputs are taken from an infinite dimensional, separable Banach space $\cX$.

\paragraph{Reconstruction error norms.} 
We will restrict our attention to the following two approximation theoretic settings.

In the first setting, we fix a compact set $\cK \subset \cX$ and the aim is to approximate an operator $\cG: \cK \to \R$, uniformly over $\cK$. Thus, we study approximation in spaces of non-linear functionals $C(\cK) = C(\cK;\R)$, and the approximation error of an encoder/decoder pair $(\cE,\cD)$ is measured by
\[
\Vert \cG - \cD\circ \cE(\cG) \Vert_{C(\cK)}
= 
\sup_{u\in \cK} 
|\cG(u) - \cD\circ \cE(\cG)(u) |.
\]

In the second setting, we fix a Gaussian measure $\mu\in \cP(\cX)$ and we are interested in approximating operators $\cG: \cX \to \R$ with respect to the norm $L^p_\mu(\cX)$. The approximation error of an encoder/decoder pair $(\cE,\cD)$ is measured by
\[
\Vert \cG - \cD\circ \cE(\cG) \Vert_{L^p_\mu(\cX)}
= 
\E_{u\sim \mu}\Big[
|\cG(u) - \cD\circ \cE(\cG)(u) |^p
\Big]^{1/p}.
\]

\paragraph{A class of $C^k$ operators.}
To define the relevant sampling and nonlinear $n$-widths, we need to specify a class of operators $\U \subset C(\cK)$ or $\U \subset L^p_\mu(\cX)$, respectively. In the following discussion, we consider the relevant class to consist of $k$-times Fr\'echet differentiable operators. We summarize the main points here; additional discussion can be found in Appendix \ref{sec:notation}. Given $\cK \subset \cX$, we denote by 
\[
C^k (\cK) := \{\cG|_{\cK}: \cK \to \R \, |\, \cG \in C^k(\cX)\},
\]
the set of ``trace-operators'' which is obtained by restricting to $\cK$ the set of $k$-times Fr\'echet differentiable operators in $C^k(\cX)$. Here $C^k(\cX)$ is defined as the set of operators that possess globally, uniformly bounded derivatives. We will establish lower bounds on the sampling $n$-with and the nonlinear $n$-width for the unit ball $\U = \cU \big( C^k (\cK) \big )$. According to our convention, $\cU(C^k(\cK))$ consists of all nonlinear operators $\cG: \cK\subset \cX \to \R$ possessing an extension $\cG: \cX \to \R$ with 
\[
\Vert \cG \Vert_{C^k(\cX)} = \max_{0\le \ell \le k} \sup_{u \in \cX} \Vert d^\ell \cG(u) \Vert \le 1.
\]

\paragraph{Assumptions on $\cK$} In order to quantify the effect of working in the
infinite dimensional setting, we must ensure that the set $\cK$ is large enough. Note that if $\cK \cong [0,1]^d$ then $C^k(\cK) \cong C^k([0,1]^d)$
and the results of subsection~\ref{subsubsec:lower_compact} apply. Therefore $\cK$ cannot be isomorphic to any finite dimensional space and must be truly infinite dimensional for any new interesting effect to be observed. Following \cite{lanthaler2023curse}, we make this notion precise by requiring that $\cK$ contain hypercubes of any dimension with algebraically decaying side-length. These hypercubes can be naturally related to smoothness classes when $\cX$ is a Banach function space.

\paragraph{Bounded Bi-orthogonal System.}

For any $n \in \mathbb{N}$, we say that a sequence of elements $\phi_1,\dots, \phi_n \in \cX$ and a sequence of dual elements $\phi^\ast_1, \dots, \phi^\ast_n \in \cX^\ast$ form a \textit{bi-orthogonal $M$-bounded system}, if for all $j, k \in [n]$, we have $\|\phi_j\|_\cX \leq 1$, $\phi_k^*(\phi_j) = \delta_{kj}$, and
$\|\phi_j^*\|_{\cX^*} \leq M$.

\paragraph{Hypercubes.}

We say that a set $\cK \subseteq \cX$ contains \textit{$\alpha$-hypercubes} of arbitrary dimension with  decay rate $\alpha>0$, if there exist constants $M,c>0$, such that for any $n \in \N$, there exists a bi-orthogonal $M$-bounded system with elements $\phi_1,\dots, \phi_n \in \Omega$ such that
\[
\bigg \{ c n^{-\alpha} \sum_{j=1}^n y_j \phi_j : y_1,\dots, y_n \in [0,1] \bigg \} \subseteq \cK.
\]

The next proposition shows that prototypical compact sets on function spaces which are defined by a smoothness constraint, contain  $\alpha$-hypercubes in the above sense. We focus here on $L^p$-based Sobolev functions $W^{s,p}$ and continuously differentiable input functions in $C^s$. Extension of this observation to other spaces, such as Besov spaces, is also possible.

\begin{proposition}
\label{prop:alpha-cubes}
Let $D \subset \R^d$ be a compact Lipschitz domain and fix $s\in \N_{>0}$, $p\in [1,\infty)$. If $\cK = \cU \big( W^{s,p}(D) \big )$ viewed as a subset of $\cX = L^p(D)$ then $\cK$ contains an $\alpha$-hypercube with decay rate $\alpha = s/d + 1$. If $\cK= \cU \big ( C^s(D) \big )$ viewed as a subset of $\cX = C(D)$, then $\cK$ contains an $\alpha$-hypercube with decay rate $\alpha = s/d$.
\end{proposition}

The details of the proof of Proposition \ref{prop:alpha-cubes} are given in Appendix \ref{sec:alpha-cubes}.

\begin{remark}
The decay rate for $\cK = \cU \big( W^{s,p}(D) \big )$ could potentially be improved to $\alpha = s/d$, based on a wavelet characterization of Sobolev spaces. We do not pursue this refinement here, contending ourselves with the rate $\alpha = s/d + 1$ which is derived based on an elementary argument involving Fourier series.
\end{remark}

\subsubsection{Sample Complexity of Operator Learning on Compact Sets}

We first consider the following setting.
Given a class of non-linear, $k$-times Fr\'echet differentiable, functionals $\cG: \cK \subset \cX \to \R$, with $\cK$ compact, we aim to approximate any such $\cG$ from data pairs $\{u_j, \cG(u_j)\}$ uniformly over $\cK$, i.e. with respect to the uniform norm
\begin{align}
\label{eq:ol-uniform}
\Vert \cG \Vert_{C(\cK)} := \sup_{u\in \cK} |\cG(u)|.
\end{align}

More precisely, we aim to approximate $\cG \in \cU \big ( C^k(\cK) \big )$, belonging to the unit ball of $C^k(\cK)$, uniformly over the compact set $\cK$.
In this setting we can prove the following result on the sample-complexity in a minimax sense.

\begin{theorem}
\label{thm:ol-uniform}
    Suppose $\cK \subseteq \cX$ contains $\alpha$-hypercubes of arbitrary dimension and let $\U = \cU \big ( C^k(\cK) \big)$ for some $k \in \N$. Then there exists a constant $R = R(k, \alpha, \cK) > 0$ such that
    \begin{align}
    \label{eq:ol-uniform1}
    s_n (\U)_{C(\cK)} \geq d_n (\U)_{C(\cK)} \geq R \big ( \log (n) \big )^{-(\alpha + 1)k}.
    \end{align}
\end{theorem}

Theorem \ref{thm:ol-uniform} shows that operator learning of $C^k$-operators, with respect to the supremum norm on input functions, generally requires a number of samples $n \gtrsim \exp(c\epsilon^{-\lambda})$ scaling exponentially in the desired approximation accuracy $\epsilon$.

\begin{proof}{(Proof of Theorem \ref{thm:ol-uniform})}

The proof relies on the following lemma, which we prove in detail in Appendix \ref{sec:embedding_cube}:
\begin{lemma}
\label{lem:embedding_cube}
Suppose $\cK \subseteq \cX$ contains $\alpha$-hypercubes of arbitrary dimension with decay rate $\alpha>0$. Then for any $k, d \in \N$, there exists a linear embedding $\iota_d: C^k_0 ([0,1]^d) \embeds C^k(\cK)$ such that, for all $f \in C^k_0([0,1]^d)$, 
\[
\Vert \iota_d f \Vert_{C^k(\cK)} \le R d^{(\alpha+1)k} \Vert f \Vert_{C^k([0,1]^d)},
\]
for some constant $R = R(k,\cK) > 0$, and
\[
\Vert \iota_d f \Vert_{C(\cK)} \ge \Vert f \Vert_{C([0,1]^d)}.
\]
Furthermore, there exists a continuous, linear mapping $h_d : [0,1]^d \to \cK$ such that,
for all $f \in C^k_0([0,1]^d)$ and $y \in [0,1]^d$,
\[f(y) = \big ( \iota_d f \big ) \big ( h_d (y) \big ).\]
\end{lemma}

Lemma \ref{lem:embedding_cube} allows us to ``embed'' a finite-dimensional approximation problem in the infinite-dimensional approximation problem of Theorem \ref{thm:ol-uniform}. This allows us to use the lower bounds on the continuous non-linear $n$-widths of Theorem \ref{thm:continous_width} combined with an optimally chosen embedded dimension $d$ to prove Theorem \ref{thm:ol-uniform}.

Recall that $\U = \cU(C^k(\cK))$. Assuming Lemma \ref{lem:embedding_cube}, we will now derive \eqref{eq:ol-uniform1}. The inequality $s_n (\U)_{C(\cK)} \geq d_n (\U)_{C(\cK)}$ follows directly by continuity of the delta functionals.
    Let $\cE_n : C(\cK) \to \R^n$ be an arbitrary mapping such that $\cE_n|_\U$ is continuous and $\cD_n : \R^n \to C(\cK)$ be an arbitrary mapping. Let $\iota_d : C^k_0 ([0,1]^d) \to C(\cK)$ and $h_d : [0,1]^d \to \cK$ be the mappings from Lemma~\ref{lem:embedding_cube}. For any $f \in C^k_0([0,1]^d)$, define $\tilde{\cE}_n : C^k_0([0,1]^d) \to \R^n$ by $\tilde{\cE}_n (f) = \cE_n (\iota_d f)$. Furthermore, for any $w \in \R^n$, define
    $\tilde{\cD}_n : \R^n \to C([0,1]^d)$ by $\tilde{\cD}_n (w) = \cD_n (w) ( h_d (\cdot) )$.
    From these definitions, it immediately follows that 
    \[(\tilde{\cD}_n \circ \tilde{\cE}_n)(f)(y) = (\cD_n \circ \cE_n)(\iota_d f)(h_d(y))\]
    for any $f \in C^k_0([0,1]^d)$ and $y \in [0,1]^d$. Thus, by Lemma~\ref{lem:embedding_cube}, we find
    \begin{align*}
         |f(y) - (\tilde{\cD}_n \circ \tilde{\cE}_n)(f)(y)| = | (\iota_d f)(h_d(y)) - (\cD_n \circ \cE_n)(\iota_d f)(h_d(y))|.
    \end{align*}
    Denote $\cB^{k,d}_0 = \cU \big ( C^k_0 ([0,1]^d) \big )$ then
    \[
    \sup_{f \in \cB^{k,d}_0} \sup_{y \in [0,1]^d} |f(y) - (\tilde{\cD}_n \circ \tilde{\cE}_n)(f)(y)| \leq \sup_{\cF \in \iota_d (\cB^{k,d}_0)} \sup_{u \in \cK} |\cF (y) - (\cD_n \circ \cE_n)(\cF)(y)| 
    \]
    follows by the inclusion $h_d ([0,1]^d) \subseteq \cK$. Note that, for any $\cF \in \iota_d (\cB^{k,d}_0)$, there exists $f \in \cB^{k,d}_0$ such that $\cF = \iota_d f$. Then by Lemma~\ref{lem:embedding_cube}, we have
    \[
    \|\cF\|_{C^k(\cK)} \leq R d^{(\alpha + 1)k} \|f\|_{C^k([0,1]^d)} \leq R d^{(\alpha + 1)k}
    \]
    for some constant $R = R(k,\cK) > 0$. It follows that $\iota_d \big ( \cB^{k,d}_0 \big ) \subseteq \rho \U$ with $\rho = R d^{(\alpha + 1)k}$. Therefore
    \[
    \sup_{f \in \cB^{k,d}_0} \sup_{y \in [0,1]^d} |f(y) - (\tilde{\cD}_n \circ \tilde{\cE}_n)(f)(y)| \leq \sup_{\cF \in \rho \U} \sup_{u \in \cK} |\cF (y) - (\cD_n \circ \cE_n)(\cF)(y)|.
    \]
    Taking the infimum over all $\cE_n, \cD_n$, we find  
    \[
    \inf_{\cE_n, \cD_n} \sup_{f \in \cB^{k,d}_0} \sup_{y \in [0,1]^d} |f(y) - (\tilde{\cD}_n \circ \tilde{\cE}_n)(f)(y)| \leq \rho d_n  (\U)_{C(\cK)}.
    \]
    Since $\tilde{\cE}_n, \tilde{\cD}_n$ are specific instances of encoder, decoder pairs build from $\cE_n,\cD_n$, we find 
    \[d_n \big( \cB^{k,d}_0 \big )_{C([0,1]^d)} \leq \inf_{\cE_n, \cD_n} \sup_{f \in \cB^{k,d}_0} \sup_{y \in [0,1]^d} |f(y) - (\tilde{\cD}_n \circ \tilde{\cE}_n)(f)(y)|\]
    and therefore
    \[d_n  (\U)_{C(\cK)} \geq \rho^{-1} d_n \big( \cB^{k,d}_0 \big )_{C([0,1]^d)}.\]
    Note that Theorem~\ref{thm:continous_width} holds for $\cB^{k,d}_0$ with the same proof, therefore
    \[d_n  (\U)_{C(\cK)} \geq R d^{-(\alpha + 1)k} n^{-k/d}\]
    where we have re-defined $R$ to absorb the relevant constant. The above inequality holds for any $d \in \N$, 
    we can therefore optimize by choosing $d = k (\alpha+1)^{-1} \log (n)$ to obtain,
    \[d_n  (\U)_{C(\cK)} \geq R \big ( \log (n) \big )^{-(\alpha + 1)k},\]
    where we have, again, similarly re-defined $R$. This is the claimed lower bound on the operator learning $n$-width.

\end{proof}

\subsubsection{Sample Complexity of Operator Learning in Expectation}

Theorem \ref{thm:ol-uniform} exhibits an exponential lower bound on the sample complexity of operator learning in the setting where the goal is to approximate operators uniformly over a compact set $\cK$. One may wonder whether more benign bounds could be achieved if the approximation error was measured in a mean squared distance, or a more general $L^p_\mu$ distance with respect to a measure $\mu$, instead of the uniform $C(\cK)$-norm \eqref{eq:ol-uniform}. 

The next theorem shows that the fundamental curse of data-complexity persists even with respect to a weaker $L^p_\mu$ norm. To this end, we next consider a Gaussian measure $\mu$ on a Banach function space $\cX$, with at most algebraically decreasing eigenvalues, $\lambda_j \gtrsim j^{-\alpha}$, of the covariance operator.
We measure the approximation error with respect to the $L^p_\mu$-norm,
\begin{align}
\label{eq:ol-gaussian}
\Vert \cG \Vert_{L^p_\mu} := \mathbb{E}_{u\sim \mu}[|\cG(u)|^p]^{1/p}.
\end{align}
The following theorem shows that the sampling $n$-width scales only logarithmically in the number of samples $n$.

\begin{theorem}
\label{thm:ol-gaussian}
    Assume the setting described above and let $\U = \cU \big ( C^k(\cX) \big)$ for some $k \in \N$. Then, for any $1 \leq p < \infty$, there exists a constant $R = R(k, p, \alpha, \cK) > 0$ such that
    \begin{align}
    \label{eq:ol-gaussian1}
    s_n (\U)_{L^p_{\mu}(\cX)} \geq R \big ( \log (n) \big )^{-(\alpha + 3)k}.
    \end{align}
\end{theorem}

\begin{remark}
We note that the above theorem provides a lower bound only on the sampling width $s_n(\U)$. We have not managed to derive a similar lower bound on the continuous non-linear $n$-widths $d_n(\U)$ in this setting. Our tightest bound exhibits a scaling of the form $n^{-1/p} \log(n)^{-\lambda}$ (cp. Theorem \ref{thm:lp-gauss} in Appendix \ref{sec:ol-gaussian}). 
\end{remark}

Theorem \ref{thm:ol-gaussian} shows that operator learning of $C^k$-operators, in the $L^p_\mu$-norm with respect to random input functions drawn from a Gaussian measure $\mu$, generally requires a number of samples $n \gtrsim \exp(c\epsilon^{-\lambda})$ scaling exponentially in the desired approximation accuracy $\epsilon$.

The proof of Theorem \ref{thm:ol-gaussian} is similar to the proof of Theorem \ref{thm:ol-uniform}, relying on the finite-dimensional lower bound implied by \ref{thm:sampling_width_gaussian} instead of \ref{thm:continous_width}. We refer to Appendix \ref{sec:ol-gaussian} for additional details.

\section{Upper data complexity bounds}
\label{sec:upper}

In the previous section, we derived lower bounds on the sampling widths and non-linear $n$-widths of spaces of operators characterized merely by their $C^k$-regularity. In particular, we show that, in this generality, operator learning suffers from a curse of data-complexity, generally requiring a number of data pairs which is exponential in the inverse of the desired accuracy $\epsilon$. 

The present section will discuss a setting where more benign bounds on the data-complexity are possible, in the context of operator learning. The purpose of this section is to show that \emph{parameter-efficiency}  implies \emph{data-efficiency}. More precisely, we aim to derive ``efficient'' upper bounds on the amount of data required to approximate a given operator, when we restrict attention to only those operators that allow for approximation by a class of neural operators with moderate model size. Here, we posit that efficient data-complexity bounds should require a number of data pairs, $n \lesssim \epsilon^{-1/\lambda}$, which scales at most algebraically with the desired accuracy $\epsilon$, for some fixed $\lambda>0$. The parameter $\lambda$ is the convergence rate in terms of available data-pairs, since the error then decays as $\epsilon \lesssim n^{-\lambda}$. Similarly, we posit that an operator allows for efficient approximation by a class of neural operators $\Psi$, if the required model size to achieve given accuracy $\epsilon$ scales at most algebraically with respect to $\epsilon$, i.e. $\size(\Psi) \lesssim \epsilon^{-1/\gamma}$. Here, $\gamma$ corresponds to the theoretically achievable approximation rate in terms of model size. We aim to quantify the relation between the model convergence rate $\gamma$ and the data convergence rate $\lambda$.

\paragraph{Setting.}
Rather than trying to propose a completely general theory, we will develop these ideas for the special case of Fourier neural operators (FNO), approximating a relevant class of operators,
\[
\cG: \cK \subset L^2(D) \to L^2(D),
\]
mapping square-integrable input functions in a set $\cK\subset L^2(D)$ to square-integrable output functions. For simplicity and due to certain restrictions of the FNO architecture, the underlying domain $D = [0,1]^d$ is taken to be the unit cube in $d$ spatial dimensions where typically $d\in \{1,2,3\}$. Even though we restrict attention to this particular class of neural operators, we point out that the ideas that will be developed are much more general, and can readily be adapted to other frameworks, such as DeepONets \cite{lu2019deeponet}, PCA-Net \cite{pca}, and variants of these frameworks.

As will be explained in further detail in the subsequent sections, we will consider the following setting. We assume we are given:
\begin{enumerate}[(1)]
\item A compact set $\cK \subset L^2(D)$ of input functions on $D = [0,1]^d$,
\item A probability measure $\mu \in \cP(L^2(D))$, supported on $\cK$, i.e.
\[
\supp(\mu) \subset \cK,
\]
\item An operator $\cG \in \cA^\gamma$ belonging to a class of operators to be specified below.
\end{enumerate}
The question to be addressed is how many data pairs $\{u_j, \cG(u_j)\}_{j=1}^n$ are needed to approximate $\cG$ to a prescribed accuracy $\epsilon$?

In the following, we will first review the FNO architecture, then define relevant spaces $\cA^\gamma$ consisting of operators $\cG$ that are ``efficiently approximated'' by FNO, and finally proceed to estimate the sampling widths and nonlinear $n$-widths of these spaces.

\subsubsection{Fourier neural operator (FNO)}

Fourier neural operator is an operator learning framework proposed in \cite{fourierop2020}. We first recall the notion of Fourier neural operators. 

\paragraph{Architecture}

Let $\cZ = \cZ (D; \R^{d_{\mathrm{in}}})$ and $\cW = \cW(D; \R^{d_{\mathrm{out}}})$ be two Banach function spaces, consisting of functions $u: D \to \R^{d_\mathrm{in}}$ and $w: D \to \R^{d_{\mathrm{out}}}$, respectively. A Fourier neural operator (FNO) defines a nonlinear operator 
\[
\Psi: \cZ (D; \R^{d_{\mathrm{in}}}) \to \cW(D; \R^{d_{\mathrm{out}}}),
\]
mapping between these spaces.
By definition of the FNO architecture, such $\Psi$ takes the form 
\begin{align}
\label{eq:FNO}
\Psi(u;\theta) = Q \circ \cL_L \circ \dots \circ \cL_1 \circ P(u). 
\end{align}
where $P: u(x) \mapsto Pu(x)$ is a linear, lifting layer, $Q: v(x) \mapsto Qv(x)$ is a linear, projection layer, and the $\cL_\ell: \cV(D;\R^{\dc}) \to \cV(D;\R^{\dc})$ are the hidden layers, mapping between hidden states $v \mapsto \cL_\ell(v) \in \cV(D;\R^{\dc})$. The hidden states are vector-valued functions with $\dc$ components,  $v: D \to \R^{\dc}$,  belonging to a Banach function space $\cV(D;\R^\dc)$. Here, the ``channel width'' $\dc$ is a hyperparameter of the architecture. Each hidden layer $\cL_\ell$ is of the form
\[
\cL_\ell(v)(x) 
:=
\sigma \big (
Wv(x) + Kv(x) + b(x)
\big )
\]
where $W \in \R^{\dc\times \dc}$ is a matrix multiplying $v(x)$ pointwise. $K$ is a non-local operator of the form
\[
v(x) \mapsto (Kv)(x) := \cF^{-1} \big ( \hat{P}_k \cF v(k) \big ) (x),
\]
with $\cF$ (and $\cF^{-1}$) the Fourier transform (and its inverse). The matrix $\hat{P}_k \in \C^{\dc \times \dc}$ is a tunable Fourier multiplier indexed by $k\in \Z^d$. It is assumed that $\hat{P}_k \equiv 0$ for $|k|_{\ell^\infty}\ge \kappa$, i.e. for wavenumbers $k$ above a specified Fourier cut-off parameter $\kappa$. This Fourier cut-off $\kappa$ is a second hyperparameter of the FNO architecture. We collect the values for different $k\in \Z^d$, $|k|_{\ell^\infty} < \kappa$, in a tensor $\hat{P} = \{ \hat{P}_k \}_{|k|_{\ell^\infty}< \kappa} \in \C^{(2\kappa-1)^d \times \dc \times \dc}$, which acts on the Fourier coefficients $\hat{v}(k) = \cF(v)(k)$, by 
\[
(\hat{P} \hat{v})(k)_i := \sum_{j=1}^\dc \hat{P}_{k,ij} \hat{v}(k), \quad (k \in \Z^d, \; |k|_{\ell^\infty}<\kappa).
\]
Finally, the bias functions are assumed to be of the form
\[
b(x) = \sum_{|k|_{\ell^\infty}< \kappa} \hat{b}_k e^{ikx}, \quad \hat{b}_k \in \C^\dc.
\]
The resulting FNO architecture depends on the channel width $\dc$, Fourier cut-off parameter $\kappa$ and depth $L$. 
\begin{table}
\centering
\begin{tabular}{c|l}
symbol & meaning \\
\hline
$\dc$ & channel width \\
$\kappa$ & Fourier cut-off \\
$L$ & depth \\
$d_\theta$ & total number of parameters \\
$B$ & parameter bound, $\Vert \theta \Vert_{\ell^\infty} \le B$ \\ \hline
$\Sigma_m$ & FNOs obeying hyper-parameter bound \eqref{eq:sigman} \\
$\cA^\gamma$ & operators approximated by FNO at rate $\gamma$, cp. \eqref{eq:Agamma-def}\\
$\cU(\cA^\gamma)$ & unit ball in $\cA^\gamma$, cp. \eqref{eq:Bgamma}
\end{tabular}
\caption{Summary of (hyper-)parameters of the FNO architecture and sets of operators defined by FNO. This notation is used throughout the text.}
\label{tab:1}
\end{table}

\begin{remark}
Following the theoretical work in \cite{thfno}, we assume the lifting and projection layers to be linear. We note that in practical applications, $P$ is often replaced by a shallow neural network, resulting in a mapping $u(x) \mapsto P(x,u(x))$. Furthermore, the biases $b(x)$ are often chosen to be constant in practice; extension of our results to this alternative setting is straightforward, but will not be considered here for simplicity of the exposition and to make our discussion consistent with the analysis of \cite{thfno}. 
\end{remark}

We collect all tunable parameters in a vector $\theta \in \R^{d_\theta}$. Any parameter $\theta \in \R^{d_\theta}$ can be decomposed layer-wise, as 
\[
\theta = (\theta_{L+1}, \theta_L, \dots, \theta_1, \theta_0),
\]
where 
\[
\theta_\ell = 
\set{
W^{(\ell)}_{ij}, \hat{P}_{k,ij}^{(\ell)}, \hat{b}_k^{(\ell)}
}{
i,j = 1,\dots, \dc, \, |k| < \kappa, \, k\in \Z^d
},
\]
collects the parameters of the $\ell$-th hidden layer, for $1\le \ell \le L$. We denote by $\theta_0 = \set{P_{ij}}{i,j = 1,\dots, \dc}$ the parameters of the projection $P$ and by $\theta_{L+1} = \set{Q_{ij}}{i,j = 1,\dots, \dc}$ the parameters of lifting $Q$. Assuming that $d_{\mathrm{in}}, d_{\mathrm{out}} \le d_c$, the dimension of $\theta \in \R^{d_\theta}$ is upper bounded by 
\begin{align}
\label{eq:dtheta}
d_\theta 
\le 
\dc d_{\mathrm{in}} + L(\dc^2 + (2\kappa)^d \dc^2 + 2\kappa \dc) + \dc d_{\mathrm{out}}
\le
5 (2\kappa)^d L \dc^2.
\end{align}
Consistent with practical implementations, we will assume throughout that the hidden channel dimension of the FNO is at least as large as both the input and output dimensions $d_{\mathrm{in}}, d_{\mathrm{out}}$. We include a list of hyperparameters in Table \ref{tab:1} to aid clarify notation.

\subsubsection{FNO approximation spaces}
\label{sec:fno-spaces}

In this section, the goal is to introduce certain spaces of operators $\cA^\gamma$ that are efficiently approximated by FNO, at rate $\gamma > 0$. Roughly speaking, the space $\cA^\gamma$ is the most natural alternative to the space of $C^k$-differentiable operators in the context of operator learning with FNO. In contrast to the approximation of operators $\cG \in C^k$, which is not tractable in general, the approximation of operators $\cG \in \cA^\gamma$ is tractable by definition. Thus, $\cA^\gamma$ can be interpreted as a space of operators that allow, in principle, for efficient approximation by FNO. Our discussion is inspired by \cite{grohs2023proof}, where similar spaces are defined for finite-dimensional neural networks.

\paragraph{Definition of $\cA^\gamma$.} 
We first define a set of operators $\Sigma_m$ as the set of all FNOs with Fourier cut-off $\kappa$, hidden channel dimension $\dc$, depth $L$ and uniform weight bound $\Vert \theta \Vert_\infty \le B$, satisfying the bounds:
\begin{align}
\label{eq:sigman}
\kappa^d, \dc, L \le m, \quad B \le \exp(m).
\end{align}

Given the sets $\Sigma_m$ defined by \eqref{eq:sigman}, we next define $\cA^\gamma$ as the set of all continuous operators $\cG: \cK \subset L^2(D) \to L^2(D)$, which can be approximated by FNOs at approximation rate $\gamma >0$. More precisely, by definition we let $\cG \in \cA^\gamma$ if, and only if, $\cG \in C \big ( \cK;L^2(D) \big )$ and if there exists a constant $C_\cG \ge 0$, such that 
\begin{align}
\label{eq:Agamma-def}
\inf_{\Psi \in \Sigma_m} \Vert \cG - \Psi \Vert_{C \left ( \cK;L^2(D) \right )} \le C_\cG m^{-\gamma}, 
\quad
\forall \, m\in \N,
\end{align}
where we recall that 
\[
\Vert \Psi - \cG \Vert_{C(\cK;L^2(D))} 
:= 
\sup_{u\in \cK} \Vert \Psi(u) - \cG(u) \Vert_{L^2(D)}.
\]
We will define the following quantity on $\cA^\gamma$:
\[
|\cG |_{\cA^\gamma} := \sup_{m\in \N} \left\{
m^\gamma \inf_{\Psi \in \Sigma_m} \Vert \cG - \Psi \Vert_{C(\cK;L^2(D))}
\right\}.
\]
This quantity can be loosely thought of as a pseudo-seminorm. Technically, the homogeneity condition only holds approximately, so it is not a true pseud-seminorm (see \cite{grohs2023proof} for extended discussion in the finite-dimensional case). On the space $\cA^\gamma$ we introduce the following quantity, which we loosely think of as the corresponding ``pseudo-norm'':
\[
\Vert \cG \Vert_{\cA^\gamma} := \Vert \cG \Vert_{C(\cK;L^2(D))} + |\cG|_{\cA^\gamma}.
\]
The class $\cA^\gamma$ is arguably the relevant class of operators $\cG$, for which meaningful approximation theory for the FNO architecture can be developed. 
We will denote by $\cU(\cA^\gamma)$ the unit ball in $\cA^\gamma$, i.e.
\begin{align}
\label{eq:Bgamma}
\cU(\cA^\gamma) := \set{\cG \in \cA^\gamma}{\Vert \cG \Vert_{\cA^\gamma} \le 1}.
\end{align}
Our ultimate aim is to derive upper bounds on the data-complexity of operator learning with FNO over this set.

\begin{remark}
\label{rem:molla}
Our understanding of the class $\cA^\gamma$ introduced above
is still very limited. In particular, we are far from a useful characterization of $\cA^\gamma$ or even subsets thereof. Nevertheless, in specific settings, it is known that solution operators of certain PDEs, such as the solution operator of the Navier-Stokes equations, and the solution operator of the Darcy flow problem, belong to some $\cA^\gamma$ (cp. \cite[Thm. 26 and Thm. 28]{thfno}). These results can likely be generalized to many other PDE solution operators. At present, it remains unclear to which extent a general theory, including a (partial) characterization of the condition $\cG \in \cA^\gamma$, can be achieved even in the context of PDE solution operators.
\end{remark}

Given the last Remark \ref{rem:molla}, we will not attempt to further characterize $\cA^\gamma$. Our goal in the following is instead to take $\cA^\gamma$ (or the unit ball $\cU(\cA^\gamma)$) as our \emph{definition} of the relevant set of operators of interest, and investigate only the data-complexity of operator learning in this class. Note that the definition of $\cA^\gamma$ is completely independent of any data; a priori, assuming that $\cG \in \cA^\gamma$ only guarantees the existence of good approximations, but does not guarantee that good approximations can be found from data pairs $\{u_j,\cG(u_j)\}_{j=1}^n$.

\paragraph{Approximation theoretic setting.}
As already mentioned above, we will be interested in the approximation of continuous operators $\cG: \cK \subset L^2(D) \to L^2(D)$, belonging to $\cA^\gamma$. Throughout this discussion, we will fix a probability measure $\mu \in \cP \big( L^2(D) \big )$, with support $\supp(\mu) \subset \cK$. Our notion of ``efficiency'', to be made precise below, will be related to approximation of the underlying operator $\cG\in \cA^\gamma$ with respect to the Bochner $L^2_\mu$-norm:
\[
\Vert \Psi - \cG \Vert_{L^2_\mu} := \E_{u\sim \mu}
\left[ 
\Vert \Psi(u) - \cG(u) \Vert_{L^2(D)}^2 
\right]^{1/2}.
\]
Measuring the error in the $L^2_\mu$-norm is motivated by the practical application of FNOs, which are usually optimized by minimizing a mean-squared error (MSE) over an input distribution. In the following discussion, we fix the compact set $\cK$ and the probability measure $\mu$ with $\supp(\mu) \subset \cK$. An important ingredient of our analysis will be the interplay between the spaces $C \big ( \cK;L^2(D) \big )$ and $L^2_\mu \big ( \cK;L^2(D) \big )$.

\subsubsection{Data-complexity bound for FNO}

In the first part of this work, we have shown that any deterministic, query-based algorithm needs an exponential number of queries to approximate arbitrary $k$-times continuously differentiable operators in the unit ball in $C^k(\cK;\cY)$ even for $\cY = \R$. The goal of this section is to derive upper data-complexity bounds when the unit ball in $C^k$ is replaced by the set of efficiently approximated operators $\cU(\cA^\gamma)$ of the last section.

\paragraph{Main result.} We first state our main result, and then outline the general strategy to derive this result. After this outline, we will proceed to provide the details of our derivation. Throughout the following discussion, we continue to denote by $\cK\subset L^2(D)$ the compact set of inputs, and by $\mu$ be the probability measure on $L^2(D)$ with $\supp(\mu)\subset \cK$.
Our main result is the following bound on the data-complexity of operator learning on $\cU(\cA^\gamma)$.
\begin{theorem}
\label{thm:fno}
There exists a constant $C = C(d,\cK,\gamma)>0$ with the following property:
For any $n\in \N$, there exist evaluation points $u_1,\dots, u_n\in L^2(D)$ defining an encoder $C \big ( \cK;L^2(D) \big ) \to [L^2(D)]^n$ by sampling, and an associated decoder $\cD_n: [L^2(D)]^n \to \Sigma_m\subset C \big( \cK;L^2(D) \big )$ for some $m\in \N$, such that
 \[
\sup_{\cG \in \cU(\cA^\gamma)}
\Vert \cG -\cD_n \big ( \cG(u_1),\dots, \cG(u_n) \big )\Vert_{L^2_\mu}^2
\le C n^{-\frac{1}{2(1+\lambda)}},
 \]
 where $\lambda := 8\gamma^{-1}$ depends only on $\gamma$.
\end{theorem}

\begin{remark}
To prove Theorem \ref{thm:fno}, the decoder $\cD_n$ will be identified explicitly; it is obtained via empirical risk minimization. We will refer to this map as the \textit{ERM decoder}. In our proof, we will furthermore show that suitable evaluation points $u_1,\dots, u_n$ can be obtained as i.i.d. samples from $\mu$.  
\end{remark}

\begin{remark}
It may be of further interest whether it is possible to obtain similar approximation rates when measuring the error as a uniform error over $\cK$, rather than in the $L^2_\mu$-norm with respect to a measure $\mu$. In particular, whether one can achieve an upper bound of the form,
\[
\sup_{\cG \in \cU(\cA^\gamma)} 
\Vert \cG - \cD_n \big ( \cG(u_1),\dots, \cG(u_n) \big ) \Vert_{C(\cK;L^2(D))} \lesssim n^{-1/2(1+\lambda)}?
\]
Unfortunately, there is little hope to achieve such error rates in the sup-norm; indeed, this uniform approximation problem has been shown to suffer from a curse of dimensionality in a finite-dimensional setting \cite[Thm. 1.1]{grohs2023proof}, requiring $n \gtrsim \epsilon^{-d}$ samples to achieve accuracy $\epsilon$ in $d$ dimensions  (see also \cite{berner2022learning,abdeljawad2023sampling}). In the present infinite-dimensional case, which roughly corresponds to the scaling limit $d\to \infty$, we can therefore not expect to achieve any algebraic rates.
\end{remark}

If we restrict attention to the subspace $\cA^\gamma_0 \subset \cA^\gamma$ consisting of operators $\cG$ with constant-valued output functions, then we can identify $\cU(\cA^\gamma_0)$ canonically with a set of functionals $\cG: \cK \to \R$. In this case, the encoding $\cG \mapsto (\cG(u_1),\dots, \cG(u_n))$ takes values in $\R^n$, and we immediately obtain the following theorem as a simple consequence of Theorem \ref{thm:fno}.
\begin{corollary}
\label{cor:fno-functional}
There exists a constant $C = C(d,\cK,\gamma)>0$, such that for any $n\in \N$, there exist evaluation points $u_1,\dots, u_n\in L^2(D)$, and ERM decoder $\cD_n: \R^n \to \Sigma_m\subset C(\cK;\R)$ for some $m\in \N$, such that
\[
\sup_{\cG \in \cU(\cA^\gamma)} \Vert \cG - \cD_n(\cG(u_1),\dots, \cG(u_n))\Vert_{L^2_\mu}
\le C n^{-\frac{1}{2(1+\lambda)}},
\]
where $\lambda = 8 \gamma^{-1}$ depends only on $\gamma$. In particular, the sampling $n$-width $s_n$ and continuous non-linear $n$-width $d_n$ of $K = \cU(\cA^\gamma_0)$ are upper bounded by 
\[
d_n(K) \le s_n(K) \lesssim n^{-\frac{1}{2(1+\lambda)}}.
\]
\end{corollary}

It is interesting to note that, even in the limit $\gamma \to \infty$, we only reach the asymptotic Monte-Carlo rate $\sim n^{-1/2}$. Thus, there is an absolute bound on the achievable data-complexity rate, even when the convergence rate $\gamma$ (in terms of FNO model size) is arbitrarily high. This observation of a limited rate, even when $\gamma \to \infty$, is consistent with the rigorous theory-to-practice gap proved in \cite{grohs2023proof} for ReLU neural networks, which indicates that this is a fundamental limitation and not a deficit of our analysis. The extension of this theory-to-practice gap to the operator learning setting is the subject of upcoming work \cite{operatorgap2024}.

\subsubsection{Outline of the proof of Theorem \ref{thm:fno}}

We next outline the elements of the proof of Theorem \ref{thm:fno}. Our proof is based on an error estimate for minimizers of the empirical risk over a relevant subset $\Sigma_m' \subset \Sigma_m$ in the FNO model classes.  
First, we will briefly review the definition of the empirical risk minimizer (ERM). Next, we will provide a general estimate that bounds the error of the ERM over $\Sigma_m'$ in terms of the metric entropy of $\Sigma_m$. Finally, we will estimate the metric entropy of $\Sigma_m$, resulting in the claimed bound of Theorem \ref{thm:fno}. We now proceed to explain this argument in further detail, with certain technical proofs delegated to the appendices. 

\paragraph{Definition of the subset $\Sigma_m'\subset \Sigma_m$.}
We define a subset $\Sigma_m'\subset \Sigma_m$ by
\[
\Sigma_m' := \set{\Psi \in \Sigma_m}{\Vert \Psi \Vert_{C(\cK;L^2(D))} \le 2}.
\]
We note that for any $\cG \in \cU(\cA^\gamma)$, if $\Psi \in \Sigma_m$ satisfies
\[
\Vert \Psi - \cG \Vert_{C(\cK;L^2(D))} \le m^{-\gamma} |\cG |_{\cA^\gamma} \le m^{-\gamma},
\]
then, since $\Vert \cG \Vert_{C(\cK;L^2(D))} \le \Vert \cG \Vert_{\cA^\gamma}\le 1$ and $m\ge 1$, $\gamma >0$, we have
\[
\Vert \Psi \Vert_{C(\cK;L^2(D))} \le \Vert \cG \Vert_{C(\cK;L^2(D))} + \Vert \Psi - \cG \Vert_{C(\cK;L^2(D))} \le 2,
\]
and hence such $\Psi$ necessarily belongs to $\Sigma_m'$. In particular, $\cU(\cA^\gamma)$ is well approximated by the relevant subset $\Sigma_m' \subset \Sigma_m$, at the same rate $\gamma > 0$, i.e.
\[
\inf_{\Psi \in \Sigma_m'} \Vert \cG - \Psi \Vert_{C(\cK;L^2(D))} \le |\cG|_{\cA^\gamma} m^{-\gamma},
\quad
\forall \, \cG \in \cU(\cA^\gamma).
\]

\paragraph{Empirical risk minimizer (ERM).}  To derive the claimed upper bound on the data-complexity, we consider the minimization of the empirical risk functional, 
\begin{align}
\label{eq:hL}
\hL(\Psi; \cG) := \frac{1}{n} \sum_{j=1}^n \Vert \Psi(u_j) - \cG(u_j) \Vert^2_{L^2(D)},
\end{align}
for certain samples $u_1,\dots, u_n\in \cK$ and a given $\cG \in \cU(\cA^\gamma)$. Given the discussion of the previous paragraph, we restrict our search over the subset $\Sigma_m' \subset \Sigma_m$, consisting of FNOs with bound $\Vert \Psi \Vert_{C(\cK;L^2(D))}\le 2$, i.e. we consider empirical risk minimizers,
\begin{align}
\label{eq:erm}
\Psi_\cG \in \argmin_{\Psi\in \Sigma_m'} \hL(\Psi;\cG).
\end{align}
We will denote by 
\begin{align}
\label{eq:cL}
\cL(\Psi;\cG) := \E_{u\sim \mu}\left[ 
\Vert \Psi_\cG(u) - \cG(u) \Vert^2_{L^2(D)}
\right],
\end{align}
the ``population risk'' over the probability measure $\mu$. We note that minimizers of the empirical risk functional \eqref{eq:hL} exist in view of the following remark.
\begin{remark}
\label{rem:compact}
The sets $\Sigma_m \subset C \big( \cK;L^2(D) \big )$, i.e. the model class of FNOs satisfying the bounds \eqref{eq:sigman}, are compact. In fact, we will show in Lemma \ref{lem:fno-Lip} that, for any fixed FNO architecture $\Psi(\theta) = \Psi(\slot;\theta)$ depending on parameters $\theta \in \R^{d_{\theta}}$, the parameter-to-operator mapping $[-B,B]^{d_\theta} \to C \big ( \cK;L^2(D) \big )$, $\theta \mapsto \Psi(\slot;\theta)$ is Lipschitz continuous. Hence, the image of the compact set $[-B,B]^{d_\theta}$ under this continuous mapping is itself compact in $C \big ( \cK;L^2(D) \big )$. The constraint \eqref{eq:sigman} in our definition of $\Sigma_m$ allows only finitely many different architectures for a given $m$; say $\Psi_j(\slot;\theta)$, with $\theta \in \R^{d_{\theta_j}}$ and $j=1,\dots, J$. Thus, 
\[
\Sigma_m = \bigcup_{j=1}^J \Psi_j
\left(
[-B,B]^{d_{\theta_j}}
\right) \subset C \big ( \cK;L^2(D) \big ) \subset L^2_\mu \big ( \cK;L^2(D) \big )
\]
is compact in both $C \big ( \cK; L^2(D) \big )$ and $L^2_\mu \big ( \cK;L^2(D) \big )$, being a finite union of compact sets.
\end{remark}

\paragraph{Error estimate for ERM.} 
The intuition behind our estimates for the empirical risk minimizer is that for sufficiently large $n$, and for independent and identically distributed samples $u_1,\dots, u_n \sim \mu$, the empirical risk \eqref{eq:hL} should be a good approximation of the population risk \eqref{eq:cL}, and in turn, the minimizer of \eqref{eq:hL} should be an almost minimizer of \eqref{eq:cL}.
The first step in our derivation is to make this intuition precise and show that it implies the upper data-complexity bound of the next proposition.
\begin{proposition}
\label{prop:erm}
Let $\gamma > 0$, and fix $\epsilon \in (0,1]$. Assume that $m \ge (2/\epsilon)^{1/\gamma}$ is an integer. Let $n\in \N$ be such that
\[
n \ge \epsilon^{-1}
 \log\left(2\cN(\Sigma_m,\epsilon)^2\right).
\]
Then there exist fixed evaluation points $u_1,\dots, u_n \in L^2(D)$, such that for any given $\cG \in \cU(\cA^\gamma)$, the empirical risk minimizer $\Psi_\cG\in \Sigma_m'$ in \eqref{eq:erm}, 
satisfies 
\[
\sup_{\cG \in \cU(\cA^\gamma)}
\Vert \Psi_\cG - \cG \Vert_{L^2(\mu)}^2
\le 
145 \epsilon.
\]
\end{proposition}
The detailed proof of Proposition \ref{prop:erm} is given in Section \ref{pf:erm}.

\paragraph{Metric entropy estimate for $\Sigma_m$.} Assuming Proposition \ref{prop:erm}, the main missing ingredient in our proof of Theorem \ref{thm:fno} is to estimate the metric entropy of $\Sigma_m$. This is the subject of the next proposition and its corollary.
\begin{proposition}[FNO covering number estimate]
\label{prop:fno-covering}
Let $\FNO_{\kappa,\dc,L,B}$ denote the set of FNOs on a $d$-dimensional domain $\Omega = [0,1]^d$, with Fourier cut-off $\kappa$, hidden channel dimension $\dc$, depth $L$ and uniform weight bound $\Vert \theta \Vert_\infty \le B$. Assume that $\dc \ge d_{\mathrm{in}}, d_{\mathrm{out}}$ and $B\ge 1$. Then there exists a constant $C = C(d) >0$, depending on the dimension of the underlying domain, but independent of the hyper-parameters $\kappa,\dc,L,B$, such that the metric entropy of $\FNO = \FNO_{\kappa,\dc,L,B}$ is upper bounded by,
\[
\log \cN(\FNO, \epsilon)_{C(\cK;L^2(D))}
\le
C \kappa^d  L^2 \dc^2 \log\left( \frac{B L \dc \kappa}{\epsilon} \right).
\]
\end{proposition}

We provide the details of the proof of Proposition \ref{prop:fno-covering} in Appendix \ref{pf:fno-covering}. The following corollary follows from the upper bounds \eqref{eq:sigman}.

\begin{corollary}
\label{cor:sigmam-covering}
Let $\cK\subset L^2(D)$ be compact. Let $\Sigma_m \subset C \big( \cK;L^2(D) \big )$ denote the set of FNOs. Then there exists a constant $C = C(\cK, d) > 0$, such that, for all $\epsilon > 0$, 
\[
\log \cN(\Sigma_m,  \epsilon)_{C(\cK;L^2(D))}
\le 
C m^{7} \log(m\epsilon^{-1}).
\]
\end{corollary}

\paragraph{Proof of Theorem \ref{thm:fno}.}
Given the results of Proposition \ref{prop:erm} and Corollary \ref{cor:sigmam-covering}, we can now prove Theorem \ref{thm:fno}.
\begin{proof}{(Proof of Theorem \ref{thm:fno})}
Let $\epsilon \in (0,1]$ be given, and let us choose an integer $m \ge (2/\epsilon)^{1/\gamma}$; in fact, we make the specific choice,
\[
m := \left\lceil (2/\epsilon)^{1/\gamma} \right \rceil,
\]
so that $m \sim (2/\epsilon)^{1/\gamma}$.
Proposition \ref{prop:erm} shows that for any $n\in \N$ with
\begin{align}
\label{eq:nch}
n \ge 
\epsilon^{-1} \log\left(2\cN(\Sigma_m,\epsilon)^2\right),
\end{align}
there exist samples $u_1,\dots, u_n \in L^2(D)$, such that the corresponding ERM decoder, 
\[
\cD_n:
 \big ( \cG(u_1),\dots, \cG(u_n) \big ) \mapsto \Psi_\cG,
\]
defined by minimization of the empirical risk over $\Sigma_m'$, satisfies 
\begin{align}
\label{eq:uacc}
\sup_{\cG \in \cU(\cA^\gamma)}
\Vert \cG - \cD_n \big ( \cG(u_1),\dots, \cG(u_n) \big ) \Vert_{L^2_\mu}^2 
\le 
145 \epsilon.
\end{align}
As a consequence of Corollary \ref{cor:sigmam-covering}, and the fact that $\epsilon^{-1} \sim m^\gamma$, there exists a constant $C = C(d,\cK,\gamma)>0$, such that
\[
\log\left(2\cN(\Sigma_m,\epsilon)^2\right)
\le 
C m^{7} \log(m).
\]
In particular, it follows that 
\[
\epsilon^{-1} 
\log\left(2\cN(\Sigma_m,\epsilon)^2\right)
\lesssim m^{\gamma + 7} \log(m) 
\sim \epsilon^{-1 -7\gamma^{-1}} \log(\epsilon^{-1})
\lesssim \epsilon^{-1 - 8\gamma^{-1}},
\]
with an implied constant $C = C(d,\cK,\gamma)$ depending only on $d$, $\cK$ and $\gamma$. Let us summarize the resulting bound:
\begin{align}
\label{eq:ente}
\epsilon^{-1} \log(2\cN(\Sigma_m,\epsilon)^2) \le C \epsilon^{-(1+8\gamma^{-1})},
\end{align}
with constant $C = C(d,\cK,\gamma)$.
In particular, it follows from \eqref{eq:ente} that for the specific choice
\[
n := \left\lceil C \epsilon^{-(1+8\gamma^{-1})} \right\rceil
\ge \epsilon^{-1} \log(2\cN(\Sigma_m,\epsilon)^2),
\]
the bound \eqref{eq:nch} is satisfied, and hence, by \eqref{eq:uacc}, $n$ samples suffice to achieve approximation accuracy,
\[
\sup_{\cG \in \cU(\cA^\gamma)}
\Vert \cG -\cD_n \big ( \cG(u_1),\dots, \cG(u_n) \big )\Vert_{L^2_\mu}^2 
\le 145 \epsilon.
\]
 The implied constant in the asymptotic relation only depends on $d,\cK,\gamma$. In particular, upon estimating $\epsilon$ in terms of $n \sim \epsilon^{-(1+8\gamma^{-1})}$, we conclude that there exists a constant $C = C(d,\cK,\gamma)>0$, such that 
 \[
\sup_{\cG \in \cU(\cA^\gamma)}
\Vert \cG -\cD_n \big ( \cG(u_1),\dots, \cG(u_n) \big )\Vert_{L^2_\mu}^2
\le C n^{-\frac{1}{1+8 \gamma^{-1}}},
 \]
or equivalently, upon introducing $\lambda := 8\gamma^{-1}$ and taking the square-root of both sides,
 \[
\sup_{\cG \in \cU(\cA^\gamma)}
\Vert \cG -\cD_n \big ( \cG(u_1),\dots, \cG(u_n) \big )\Vert_{L^2_\mu}
\le C n^{-\frac{1}{2(1+\lambda)}}.
 \]
 This is the claimed upper bound. 
\end{proof}

\section{Conclusion}
\label{sec:conclusion}

Operator learning is rapidly emerging as a new paradigm to complement traditional numerical solvers in scientific computing. 
While empirical evidence demonstrates the practical efficacy of operator learning frameworks, a complete theoretical explanation of this empirical observation is still outstanding. Several papers have studied operator learning from an approximation theoretic perspective, with the goal of bounding the total number of parameters (model size). The complementary question of the data-complexity of operator learning has received considerably less attention in the literature. 

In this paper, we have studied the sample complexity of operator learning in three main settings: (1) uniform approximation of Fr\'echet differentiable operators, (2) approximation of Fr\'echet differentiable operators in the $L^p_\mu$-norm and, (3) the sample complexity of operator learning in spaces of operators which theoretically allow for efficient approximation by FNO at moderate model size. 

Our results in the settings (1) and (2) show rigorously that operator learning on general classes of Fr\'echet differentiable operators suffers from a ``curse of sample complexity'', requiring an exponential number of samples to achieve a desired approximation accuracy. In particular, this result implies that existing bounds on the generalization error which suffer from such an exponential curse of dimensionality, as derived in \cite{liu2022deep}, are essentially optimal. Our complexity estimates generalize well-known bounds to the infinite-dimensional context and complement the notion of the ``curse of parametric complexity'' from \cite{lanthaler2023curse,lanthaler2023operator} to establish a similar curse for the sample complexity. These exponential lower bounds on the complexity of operator learning strongly suggest that a theory of operator learning should be developed on spaces of operators that are considerably smaller than the set of bounded Fr\'echet (or Lipschitz) operators. As a consequence of these lower bounds, we argue that the empirically observed efficiency of operator learning for concrete problems cannot be explained by general arguments based on notions of operator smoothness such as Lipschitz continuity or Fr\'echet differentiability. 

In the second part of this paper, we study the sample complexity of operator learning on sets of operators which theoretically allow for efficient approximation, requiring only a modest number of tunable parameters, growing at most algebraically in the inverse of the desired approximation accuracy $\epsilon$. To illustrate the main ideas, we take the Fourier neural operator (FNO) as a case study, and propose a rigorous definition of the relevant approximation spaces $\cA^\gamma$. For these spaces, we prove that a number of samples which grows at most algebraically in the inverse of the desired approximation error $\epsilon$ is sufficient for operator learning. These results hold in the root mean square sense with respect to a probability measure $\mu$ of input functions supported on a compact set.

While we restrict attention to a specific setting to derive efficient upper sample complexity bounds in the second part of this work, we expect the basic ideas outlined in present work to apply to many other operator learning frameworks. We hope these results will motivate future investigations of the relevant approximation spaces $\cA^\gamma$ and provide a new point of view on the practical successes of operator learning. 

\section*{Acknowledgment}
NBK is grateful to the Nvidia Corporation for support through full time employment. SL acknowledges funding from the Swiss National Science Foundation through Postdoc.Mobility grant P500PT-206737.

\bibliographystyle{abbrv}
\bibliography{references}

\appendix 

\section{Spaces and sets of differentiable functionals}
\label{sec:notation}

\paragraph{$C^k(\Omega)$-spaces in finite-dimension.}
Let \(\Omega \subseteq \R^d\) for some \(d \in \N_{>0}\) be a set. We denote by \(C^0(\Omega;\R) = C(\Omega;\R) = C(\Omega)\) the set of continuous, bounded functions on \(\Omega\) taking values in \(\R\). We equip \(C(\Omega)\) with the norm
\[\|f\|_\infty = \sup_{x \in \Omega} |f(x)|, \qquad f \in C(\Omega)\]
which makes it a Banach space. For any \(k \in \N_{>0}\), we denote by \(C^k (\Omega;\R) = C^k(\Omega)\) the set of bounded, \(k\)-times continuously differentiable functions on \(\Omega\) taking values in \(\R\).
We equip \(C^k(\Omega)\) with the norm
\[\|f\|_{C^k} = \max_{0 \leq |\nu|_1 \leq k} \|D^\nu f\|_\infty, \qquad f \in C^k(\Omega)\]
which makes it a Banach space. We denote by \(\nu \in \N^d\) any multi-index and by \(|\cdot|_p\) the \(p\)-th order Euclidean norm. \(D^\nu\) denotes the partial derivative operator with respect to the multi-index \(\nu\).
We write \(C^\infty(\Omega;\R) = C^\infty(\Omega)\) for the set of infinitely differentiable functions on \(\Omega\) taking values in \(\R\) and similarly \(C^\infty_c(\Omega;\R) = C^\infty_c(\Omega)\) for the subset of \(C^\infty(\Omega)\) containing all compactly supported functions. Furthermore, we denote by $C_0 (\Omega) \subset C(\Omega)$, the subset of functions which decay to zero at $\partial \Omega$.

\paragraph{Fr\'echet differentiable functionals.}
When $\cX$ is an infinite-dimensional Banach space, we say that a non-linear functional $\cG: \cX \to \R$ is Fr\'echet differentiable, if for any $u \in \cX$, there exists a bounded linear functional $d\cG(u): \cX \to \R$, such that 
\[
d \cG(u)[v] = \lim_{\delta \to 0} \frac{\cG(u+\delta v) - \cG(u)}{\delta}, \quad \forall \, v \in \cX.
\]
Furthermore, we will say that $\cG$ is continuously Fr\'echet differentiable, if $\cG$ is Fr\'echet differentiable, and the differential $\cX \to \cL(\cX;\R)$, $u \mapsto d\cG(u)$ is continuous as a mapping from $\cX$ into the space of bounded linear functionals $\cL(\cX;\R)$, equipped with the operator norm,
\[
\Vert d \cG(u) \Vert := \sup_{v \ne 0} \frac{|d\cG(u)[v]|}{\Vert v \Vert}.
\]
Higher-order differentiability is defined analogously. For example, we say that $\cG$ is twice Fr\'echet differentiable, if $\cG$ is Fr\'echet differentiable, and $d\cG: \cX \to \cL(\cX;\R)$ is differentiable; i.e. if there exists a bounded linear operator $d^2\cG(u): \cX \mapsto \cL(\cX;\R)$, such that 
\[
d^2\cG(u)[v] = \lim_{\delta \to 0} \frac{d\cG(u+\delta v) - d\cG(u)}{\delta}, \quad \forall\, v \in \cX.
\]
Upon identifying linear operators $\cX \to \cL(\cX;\R)$ in the canonical way with bounded bi-linear maps in $\cL(\cX\times \cX;\R)$, i.e. using the isomorphism,
\[
\cL(\cX;\cL(\cX;\R)) \simeq \cL(\cX \times \cX; \R),
\]
we can interpret $d^2\cG(u): \cX\times \cX \to \R$ as a bilinear operator. We say that $\cG$ is twice continuously differentiable if $u \mapsto d^2\cG(u)$ is continuous. Proceeding iteratively, the $k$-th Fr\'echet derivative $\cG$ is defined analogously. In particular, the $k$-th Fr\'echet differential is given by a bounded $k$-linear operator,
\[
d^k \cG(u): \underbrace{\cX \times \dots \times \cX}_{\text{$k$ times}} \to \R,
\]
with norm 
\[
\Vert d^k \cG(u) \Vert = \sup_{v_1 ,\dots, v_k \ne 0}
\frac{|d^k\cG(u)[v_1,\dots, v_k]|}{\Vert v_1 \Vert \dots \Vert v_k \Vert}.
\]

Consistent with the finite-dimensional case, we introduce the spaces $C^k(\cX)$ consisting of $k$-times Fr\'echet differentiable functionals $\cG: \cX \to \R$, with \emph{bounded derivatives}, i.e. such that the quantity 
\[
\Vert \cG \Vert_{C^k(\cX)} := \max_{0\le \ell \le k} \sup_{u\in \cX} \Vert d^\ell \cG(u) \Vert < \infty.
\]
The unit ball $\cU(C^k(\cX))$ is defined as the set of $\cG\in C^k(\cX)$ such that $\Vert \cG \Vert_{C^k(\cX)} \le 1$. 

\paragraph{$C^k(\cK)$ trace-functionals for $\cK\subset \cX$ compact.}
If $\cK\subset \cX$ is a compact set, then we will denote by 
\[
C^k(\cK) := \{\cG|_\cK\, |\, \cG \in C^k(\cX)\},
\]
 the restriction to $\cK$ of functionals belonging to $C^k(\cX)$. Similarly, we denote by 
 \[
 \cU(C^k(\cK)) := \{\cG|_\cK\, |\, \cG \in \cU(C^k(\cX))\},
 \]
 the corresponding ``unit ball''. We emphasize that derivatives are usually not well-defined for functionals $\cG: \cK \to \R$, when $\cK\subset \cX$ is compact, and hence, we effectively work with a set of trace-functionals, i.e. nonlinear functionals possessing an extension to all of $\cX$ with globally bounded $C^k$-norm. We also emphasize that we only use $\cU(C^k(\cK))$ to define a set of Fr\'echet differentiable operators of interest. When studying the approximation of functionals $\cG \in \cU(C^k(\cK))$ on a compact set of inputs $u\in \cK$, we will only consider the uniform norm $\Vert \cG \Vert_{C(\cK)} = \sup_{u\in \cK} |\cG(u)|$, under which $C(\cK)=C^0(\cK)$ is itself a Banach space and never norms involving higher-order derivatives.

 \section{Continuity of linear maps over pre-compact sets}
\label{sec:cont_lin_func}

We will show that over pre-compact sets, all norms are essentially equivalent which will allows us to lower bound the sampling width by the non-linear width for many useful situations. Consider two normed vectors spaces $(X,{\|\cdot\|_X})$ and $(Y,\|\cdot\|_Y)$ and suppose that $Y$ compactly embeds in $X$. Then we have the following results.

\begin{lemma}
    \label{lemma:compact_embedding}
    Let $K \subset Y$ be a bounded set. Then $K$ is pre-compact in $X$ and there exists a constant $C = C(K) > 0$ such that
    \[\|f\|_Y \leq C \|f\|_X, \qquad \forall \: f \in K.\]
\end{lemma}
\begin{proof}
    Pre-compactness of $K$ follows immediately by the compact embedding ${Y \hookrightarrow X}$. Since $K$ is bounded, there exists some constant $C_1 > 0$ such that 
    \[K \subseteq \tilde{K} \coloneqq \{f \in Y : \|f\|_Y \leq C_1\}.\]
    Consider $\partial \tilde{K} = \{f \in Y : \|f\|_Y = C_1\}$. Since this is a closed and bounded set in $Y$, it is compact in $X$. Therefore $\|\cdot\|_X$ is uniformly continuous on $\partial \tilde{K}$. Since $0 \notin \partial \tilde{K}$, there exists a constant $C_2 > 0$ such that $C_2 \leq \|f\|_X$ for all $f \in \partial \tilde{K}$. It follows that
    \[\sup_{f \in \partial \tilde{K}} \frac{\|f\|_Y}{\|f\|_X} \leq \frac{C_1}{C_2}\]
    and, in particular,
    \[\|f\|_Y \leq \frac{C_1}{C_2} \|f\|_X, \qquad \forall \: f \in \partial \tilde{K}.\]
    Now let $f \in \tilde{K}$ and assume that $f \neq 0$, noting that the desired inequality for the case $f = 0$ follows trivially since $\|0\|_Y = \|0\|_X = 0$. Define $g \coloneqq C_1 f / \|f\|_Y$ then clearly $g \in \partial \tilde{K}$
    hence
    \[\|g\|_Y \leq \frac{C_1}{C_2} \|g\|_X.\]
    Plugging in the definition of $g$ implies
    \[\|f\|_Y \leq \frac{C_1}{C_2} \|f\|_X.\]
    Since $f$ was arbitrary and $K \subseteq \tilde{K}$, the result follows.
\end{proof}
\begin{lemma}
    \label{lemma:cont_linear_map}
    Let $K \subset Y$ be a bounded set and let $E: Y \to \R^n$ be a continuous, linear map. Then there exists a constant $C = C(E,K) > 0$ such that 
    \[|E(f) - E(g)| \leq C \|f -g\|_X, \qquad \: \forall f,g \in K.\]
    In particular, $E|_K$ is continuous in $X$.
\end{lemma}
\begin{proof}
    Let $f, g \in K$. Since $E$ is continuous and linear, there exists a constant $C_1 > 0$ such that
    \[|E(f) - E(g)| \leq C_1 \|f - g\|_Y.\]
    Define $\tilde{K} \coloneqq \{f - g \in Y : f,g \in K\}$. The triangle inequality implies that $\tilde{K}$ is bounded since $K$ is bounded. Therefore Lemma~\ref{lemma:compact_embedding} implies that there exists a constant $C_2 > 0$ such that
    \[\|f-g\|_Y \leq C_2 \|f-g\|_X.\]
    Hence,
    \[|E(f) - E(g)| \leq C_1 C_2 \|f -g\|_X\]
    as desired.
\end{proof}

\section{Bounds in finite dimension}
\label{sec:bds_finite}

\subsection{Lower bounds over the Cube}
\label{subsec:widths_cube}

\subsubsection{Partition of the Cube}
\label{subsec:partition_cube}

Let $\Omega = [0,1]^d$. We define the bump function ${\varphi : \R \to \R}$ by
\[\varphi(x) \coloneqq\begin{cases}
\text{e}^{-\frac{1}{1-x^2}}, & -1 < x < 1 \\
0, & \text{otherwise}.
\end{cases}\]
For any \(0 < \gamma < 1\), consider the following family of functions \(\sigma_\gamma  : \R \to \R\) defined as
\[\sigma_\gamma (x) \coloneqq
\begin{cases}
0, & x \leq 0, \\
\text{e} \cdot \varphi \big( \frac{2}{1-\gamma}x - 1 \big ), & 0 < x < \frac{1}{2}(1-\gamma), \\
1, & \frac{1}{2}(1-\gamma) \leq x \leq \frac{1}{2}(1+\gamma), \\
\text{e} \cdot \varphi \big( \frac{2}{1-\gamma}x - \frac{\gamma-1}{1-\gamma} \big ), & \frac{1}{2}(1+\gamma) < x < 1, \\
0, & x \geq 1.
\end{cases}
\]
The non-constant terms are simply shifts and a re-scaling of the bump function so that it is supported on \([0,1-\gamma]\) or \([\gamma,1]\) respectively and so that it takes the value one at the midpoint of its support. It is easy to check that \(\sigma_\gamma\) is supported on \([0,1]\), \(\sigma_\gamma \equiv 1\) on \( \big[ \frac{1}{2}(1-\gamma),\frac{1}{2}(1 + \gamma) \big ]\), and \(\sigma_\gamma \leq 1\) on \([0,1]\). Furthermore, due to well known differentiability properties of the bump function, we have \(\sigma_\gamma \in C^\infty_c(\R)\). For any \(l \in \N\), denote by \(\sigma_\gamma^{(l)}\) the \(l\)-th derivative of \(\sigma_\gamma\) with the convention \(\sigma_\gamma^{(0)} \coloneqq \sigma_\gamma\). We can find constants \(\beta_l > 0\) such that
\[\sup_{x \in \R} |\sigma_\gamma^{(l)} (x)| \leq \frac{\beta_l}{(1-\gamma)^l}\]
where \(\beta_0 = 1\) and, generally, \(\beta_l\) depends only on \(l\) and not \(\gamma\). Define \(\phi_\gamma : \R^d \to \R\) as
\[\phi_\gamma (x) \coloneqq \prod_{j=1}^d \sigma_\gamma (x_j), \qquad x \in \R^d.\]
It is easy to see that \(\phi_\gamma\) is supported on \(\Omega\), \(\phi_\gamma \equiv 1\) on \(\big[ \frac{1}{2}(1-\gamma),\frac{1}{2}(1 + \gamma) \big ]^d\), and \(\phi_\gamma \leq 1\) on \(\Omega\). Let \(\nu \in \N^d\) be a multi-index with \(|\nu|_1 = k\). Then
\begin{align*}
    \|D^\nu \phi_\gamma\|_\infty &= \sup_{x \in \Omega} |D^\nu \phi_\gamma (x) | \\
    &\leq \prod_{j=1}^d  \sup_{x \in \Omega} |\sigma_\gamma^{(\nu_j)} (x_j)| \\
    &\leq \prod_{\nu_j \neq 0} \frac{\beta_{\nu_j}}{(1-\gamma)^{\nu_j}} \\
    &\leq \frac{1}{(1- \gamma)^k}\max_{p \in P_{k,k}} p(\beta_1,\dots,\beta_k) \\
    &\coloneqq \frac{\Gamma(k)}{(1-\gamma)^k}
\end{align*}
where \(P_{k,k}\) denotes the set of all monomials of \(k\) variables with at most degree \(k\) and leading coefficient one. The constant \(\Gamma\) depends only on \(k\) since the cardinality of \(P_{k,k}\) depends only on \(k\), but not on \(d\) or \(\gamma\). In particular, we find that, for any \(k \in \N\), there exists a constant \(\Gamma = \Gamma (k) > 0\) such that
\begin{equation}
    \label{eq:bound_phi_der}
    \|D^\nu \phi_\gamma \|_\infty \leq \frac{\Gamma}{(1-\gamma)^{k}}, \qquad 0 \leq |\nu|_1 \leq k.
\end{equation}

Let \(m \in \N\) and uniformly subdivide \(\Omega\) into cubes of side-length \(1/m\).
There are \(n \coloneqq m^d\) such cubes, which we denote \(Q_1,\dots,Q_n\) each with volume \(1/m^d = n^{-1}\).
For any \(j \in [n]\), define, 
\[q_j \coloneqq \frac{1}{2} \mathbbm{1} - n \int_{Q_j} x \: \mathsf{d}x\]
which is the vector difference between the center of mass of \(\Omega\) and \(Q_j\). We now define the functions \(\phi_{\gamma,j} \in C^\infty_c(\R^d)\) by \(\phi_{\gamma,j} (x) \coloneqq \phi_\gamma \big( m(x + q_j) \big )\) for any \(x \in \R^d\).
In particular, each \(\phi_{\gamma,j}\) is \(\phi_\gamma\) with its support shifted from \(\Omega\) to \(Q_j\).
Furthermore, from \eqref{eq:bound_phi_der}, we find that, for any \(k \in \N\) and \(j \in [n]\),
\begin{equation}
    \label{eq:bound_phi_j_der}
    \|D^\nu \phi_{\gamma,j}\|_\infty \leq \frac{\Gamma m^{|\nu|_1}}{(1-\gamma)^{k}}, \qquad 0 \leq |\nu|_1 \leq k.
\end{equation}
It is also useful to estimate the $L^p$-norms of these functions, in particular, we have
\begin{align}
\begin{split}
    \label{eq:lp-bump}
    \|\phi_{\gamma,j}\|^p_p &= \int_{Q_j} |\phi_{\gamma,j} (x)|^p \: \mathsf{d}x \\
    &= m^{-d} \int_{\Omega} |\phi_\gamma (x) |^p \: \mathsf{d}x \\
    &\geq m^{-d} \int_{\Omega \cap \{\phi_{\gamma} \equiv 1\}} 1 \: \mathsf{d}x \\
    &= n^{-1} \gamma^d.
\end{split}
\end{align}

\subsubsection{Proof of Theorem \ref{thm:sampling_width}}
    \label{pf:sampling_width}
 
\begin{proof}
    \textbf{Step 1}: First we consider the case $q = \infty$ with $K = \cU \big (W^{k,\infty} (\Omega) \big )$ or $K = \cU \big ( C^k (\Omega) \big )$. Let 
    \[f = J \sum_{j=1}^{2n} \alpha_j \phi_{\gamma, j} \]
    where $\alpha_j \in \{-1,1\}$ and $J$ is chosen such that $\|f\|_{C^k} = \|f\|_{W^{k,\infty}} \leq 1$ so that $f \in K$.
    We will first estimate $J$. Using the fact that the $\phi_{\gamma, j}$(s) have disjoint supports along with \eqref{eq:bound_phi_j_der}, we find 
    \begin{align*}
        \|f\|_{C^k} &= \max_{0 \leq |\nu|_1 \leq k} \|D^\nu f \|_\infty \\
        &= J \max_{0 \leq |\nu|_1 \leq k} \max_{j \in [2n]} \|D^\nu \phi_{\gamma, j} \|_\infty \\
        &\leq \frac{J \Gamma n^{k/d}}{(1 - \gamma)^k}.
    \end{align*}
    If $p < \infty$, we choose $\gamma = d/(kp+d)$, which gives
    \[ (1-\gamma)^{-k} = \left( 1+ \frac{d}{kp}\right)^k.\]
    If $p = \infty$, any choice of $\gamma$ will do, so we simply pick $\gamma = 1/2$ and set
    \[T =
    \begin{cases}
        \left( 1+ \frac{d}{kp}\right)^k, & 1 \leq p < \infty, \\
        2^k, & p = \infty.
    \end{cases}
    \]
    We then have
    \[\|f\|_{C^k} \leq J T \Gamma n^{k/d}.\]
    Now suppose $q < \infty$ with $K = \cU \big (W^{k,q} (\Omega) \big )$. Similarly, we have
    \begin{align*}
    \|f\|_{W^{k,q}}^q &= \sum_{0 \leq |\nu|_1 \leq k} \|D^\nu f \|_q^q \\
    &\leq \sum_{0 \leq |\nu|_1 \leq k} \|D^\nu f \|_\infty^q \\
    &\leq J^q (k+1)d^k T^{q} \Gamma^q n^{kq/d} 
    \end{align*}
    where the estimate on the sum follows as
    \begin{align}
    \label{eq:combinatorial_sum}
    \begin{split}
    \sum_{0 \leq |\nu|_1  \leq k} 1 &= \sum_{j=0}^k {d-1+j \choose d-1} \\
    &\leq (k+1) {d-1+k \choose d-1} \\
    &= (k+1) \frac{k+d-1}{k} \cdot \frac{k+d-2}{k-1} \cdots \frac{d}{1} \\
    &\leq (k+1)d^k.
    \end{split}
    \end{align}
    In particular,
    \[\|f\|_{W^{k,q}} \leq J (k+1)^{1/q} T \Gamma d^{k/q} n^{k/d}.\]
    Therefore letting $J = \big ( (k+1)^{1/q} T \Gamma d^{k/q} n^{k/d} \big )^{-1}$, for any $1 \leq q \leq \infty$ implies $f \in K$.

    \textbf{Step 2}: Now let $Y_n = \{y_1,\dots,y_n\} \subset \Omega$ be an arbitrary set of points. Let $l : [n] \to [2n]$ be such that $y_j \in \bar{\text{supp}} \: \phi_{\gamma, l(j)} $
    for any $j \in [n]$. Since the cubes $Q_j$ overlap on their boundaries, the choice of $l$ might not be unique, however, since the $\phi_{\gamma,j}$(s) are identically zero on these boundaries any arbitrary choice will do. In particular, we have that
    \[f(y_j) = J \alpha_{l(j)} \phi_{\gamma, l(j)} (y_j).\]
    Define, for any $j \in [2n]$,
    \[ \beta_j = 
    \begin{cases}
        \alpha_j, & j \in l([n]), \\
        -\alpha_j, & \text{otherwise}
    \end{cases}
    \]
    and set 
    \[g = J \sum_{j=1}^{2n} \beta_j \phi_{\gamma, j}.\]
    Clearly $g \in K$. Furthermore, for each $j \in [n]$, we have $g(y_j) = f(y_j)$ and, in particular,
    \[\delta_{Y_n}(g) = \delta_{Y_n}(f).\]
    Therefore
    \[f - g = \big ( f - (\D_n \circ \delta_{Y_n})(f) \big ) - (g - (\D_n \circ \delta_{Y_n})(g) \big )\]
    hence
    \[\|f - g\|_p \leq \| f - (\D_n \circ \delta_{Y_n})(f) \|_p + \| g - (\D_n \circ \delta_{Y_n})(g) \|_p.\]
    It follows that
    \[\max \big \{ \| f - (\D_n \circ \delta_{Y_n})(f) \|_p, \| g - (\D_n \circ \delta_{Y_n})(g) \|_p \big  \} \geq \frac{1}{2} \|f - g\|_p\]
    and therefore
    \[s_n (K)_{L^p} \geq \frac{1}{2} \|f - g \|_p.\]

    \textbf{Step 3}: It remains to estimate the difference between $f$ and $g$. Suppose first $p = \infty$ then we have
    \[\|f - g\|_\infty = 2 J \max_{\substack{j \in [2n], \\ j \notin l([n])}} \|\phi_{\gamma, j}\|_\infty = 2 J.\]
    If $p < \infty$, we have from \eqref{eq:lp-bump},
    \begin{align*}
        \|f - g\|^p_p &= 2^p J^p \sum_{\substack{j \in [2n], \\ j \notin l([n])}} \|\phi_{\gamma, j} \|^p_p \\
        & \geq 2^p J^p n^{-1} \gamma^d n \\
        &=  2^p J^p \gamma^d
    \end{align*}
    since the sum must have at least $n$ elements in the case $|l([n])| = n$. Our choice $\gamma = d/(kp+d)$, gives
    \[\gamma^{d/p} = \left( 1 + \frac{kp}{d} \right)^{d/p} \geq e^k.\]
    Putting this together by using the definition of $T$, we have, for any $1 \leq p, q \leq \infty$,
    \[\frac{1}{2} \|f-g\|_p \geq 2^{-k} (k+1)^{-1/q} \Gamma^{-1} \left ( 1 + \frac{d}{kp} \right )^{-k} d^{-k/q} n^{-k/d}\]
    as desired.
\end{proof}

\subsubsection{Proof of Theorem \ref{thm:continous_width}}

\begin{proof}
Define \(X_n \coloneqq \text{span} \{\phi_{\gamma,1},\dots,\phi_{\gamma,n}\}\) and let  \(f \in X_n\). In particular,  
\[f = \sum_{j=1}^n \alpha_j \phi_{\gamma,j} \] 
for some \(\alpha_1,\dots,\alpha_n \in \R\). Suppose first that $p = q = \infty$. Since the $\phi_{\gamma,j}$(s) have disjoint support,
\[\|f\|_\infty = \max_{j \in [n]} |\alpha_j|.\]
Then, using \eqref{eq:bound_phi_j_der}, we find
\begin{align}
    \label{eq:xn_ck_esimate}
    \begin{split}
        \|f\|_{C^k} &= \max_{0 \leq |\nu|_1 \leq k} \|D^\nu f \|_\infty \\
        &= \max_{0 \leq |\nu|_1 \leq k} \max_{j \in [n]}  |\alpha_j| \|D^\nu \phi_{\gamma,j} \|_\infty \\
        &\leq \frac{\Gamma n^{k/d}}{(1-\gamma)^k} \max_{1 \leq j \leq n} |\alpha_j| \\
        &= 2^k \Gamma n^{k/d} \|f\|_\infty
    \end{split}
\end{align}
with the choice $\gamma = 1/2$. Suppose now that $q < \infty$. Consider first, the case $p = \infty$. We have 
\begin{equation}
    \label{eq:xn_pinfinity_lowerbound}
    \|f\|^q_\infty = \left ( \max_{j \in [n]} |\alpha_j| \right )^q \geq n^{-1} \sum_{j=1}^n |\alpha_j|^q .
\end{equation}
On the other hand, when $p < \infty$, using \eqref{eq:lp-bump}, we find 
\begin{align}
    \label{eq:xn_lp_lowerbound}
    \begin{split}
    \|f\|_{p}^p &= \int_\Omega \bigg ( \bigg | \sum_{j=1}^n \alpha_j \phi_{\gamma,j} (x) \bigg | \bigg )^p \: \mathsf{d}x \\
    &= \sum_{l=1}^n \int_{Q_l} \bigg ( \bigg | \sum_{j=1}^n \alpha_j \phi_{\gamma,j} (x) \bigg | \bigg )^p \: \mathsf{d}x \\
    &= \sum_{l=1}^n \int_{Q_l} | \alpha_l \phi_{\gamma, l} (x) |^p \: \mathsf{d}x \\
    &= \sum_{l=1}^n |\alpha_l|^p \int_{Q_l} |\phi_{\gamma,l} (x) |^p \: \mathsf{d}x \\
    &\geq n^{-1} \gamma^d \sum_{j=1}^n |\alpha_j|^p.
    \end{split}
\end{align}
If $p > q$, by H{\"o}lder's inequality,
\[\|f\|_p^p \geq n^{-1} \gamma^d n^{1 - p/q} \left ( \sum_{j=1}^n |\alpha_j|^q  \right)^{p/q}.\]
Therefore, for $p \geq q$, we find
\[\|f\|_p^q \geq n^{-1} \gamma^{dq/p} \sum_{j=1}^n |\alpha_j|^q .\]
Notice that, due to \eqref{eq:xn_pinfinity_lowerbound}, the above inequality also holds for the case $p = \infty$. Using this along with \eqref{eq:bound_phi_j_der} and \eqref{eq:combinatorial_sum}, we have 
\begin{align}
    \label{eq:xn_wkq_upperbound}
    \begin{split}
    \|f\|_{W^{k,q}} &= \left ( \sum_{0 \leq |\nu|_1  \leq k} \sum_{j=1}^n \int_{Q_j} |\alpha_j|^q |D^\nu \phi_{\gamma,j}(x)|^q \: \mathsf{d}x  \right )^{1/q} \\
    &\leq \left ( \sum_{0 \leq |\nu|_1  \leq k} \frac{\Gamma^q m^{|\nu|_1 q}}{(1-\gamma)^{kq}} n^{-1} \sum_{j=1}^n |\alpha_j|^q  \right )^{1/q} \\
    &\leq \left ( \sum_{0 \leq |\nu|_1  \leq k} \frac{\Gamma^q m^{|\nu|_1 q}}{(1-\gamma)^{kq} \gamma^{dq/p}} \|f\|_p^q  \right )^{1/q} \\
    &\leq \Gamma (k+1)^{1/q} \gamma^{-d/p} (1-\gamma)^{-k} d^{k/q} n^{k/d} \|f\|_p 
    \end{split}
\end{align}
When $ p = \infty$, we choose $\gamma = 1/2$ and obtain,
\begin{equation}
    \label{eq:xn_wkq_pinfinity_estimate}
    \|f\|_{W^{k,q}} \leq 2^k \Gamma (k+1)^{1/q} d^{k/q} n^{k/d} \|f\|_p.
\end{equation}
When $p < \infty$, choosing $\gamma = d/(kp+d)$ yields
\[
(1-\gamma)^{-k} = \left( 1+ \frac{d}{kp}\right)^k,
\quad
\gamma^{-d/p} = \left( 1 + \frac{kp}{d} \right)^{d/p} \leq e^k,
\]
and therefore
\[\|f\|_{W^{k,q}} \leq  \underbrace{ \Gamma e^k (k+1)^{1/q} \left( 1+ \frac{d}{kp}\right)^k d^{k/q} n^{k/d}}_{\coloneqq J} \|f\|_p.  \]
Since $J$ is greater than the constants appearing in \eqref{eq:xn_ck_esimate} and \eqref{eq:xn_wkq_pinfinity_estimate}, we have shown that $\{f \in X_n : \|f\|_p \leq J^{-1}\} \subseteq K$ for all relevant choices of $p$ and $q$. The result therefore follows by \cite[Theorem 3.1]{devore1989optimal}.
\end{proof}

\subsection{Lower bounds over Weighted Euclidean Space}
\label{subsec:width_gaussian}

\subsubsection{Gaussian Partition of Euclidean Space}
\label{subsec:partition_rd}

We now modify the partition defined in section~\ref{subsec:partition_cube} so that that union of the support of all functions becomes $\R^d$ instead of $[0,1]^d$.
Let $\rho_d$ denote the density of the standard $d$-dimensional Gaussian measure, dropping the subscript when $d=1$. 
Define $\xi : \mathbb{R} \to (0,1)$ by
\[\xi(x) = \int_{-\infty}^x \rho(y) \: \mathsf{d}y \]
to be the cumulative distribution function of $\rho$. It is easy to see that $\xi \in C^\infty_b (\R)$
and is a homeomorphism of $\R$ into $(0,1)$. In particular, $\xi^{-1} : (0,1) \to \R$ exists and is 
continuous. Define $\xi_d : \R^d \to (0,1)^d$ by $\xi_d (x)_j = \xi(x_j)$ for any $j \in [d]$ and 
notice that $\xi_d^{-1} : (0,1)^d \to \R^d$ is given as $\xi_d^{-1} (x)_j = \xi^{-1}(x_j)$.
The mapping $\xi_d$ is the Knothe-Rosenblatt rearrangement between the standard Gaussian on $\R^d$ and the uniform 
measure on $[0,1]^d$ which is simply the Lebesgue measure. In particular, $(\xi_d)_\sharp \rho_d (x) \: \mathsf{d}x = \mathsf{d}x$.
Define the family of functions $\tilde{\phi}_{\gamma,j} (x)  = \phi_{\gamma,j}(\xi_d(x))$ for any $j \in [n]$ and notice that 
$\tilde{\phi}_{\gamma,j}$ is supported on $\xi_d^{-1}(Q_j) \coloneqq E_j$. Each cube $Q_j \subset [0,1]^d$ is stretched to the hyper-rectangle $E_j \subset \R^d$ whose volume under the standard 
Gaussian equals the volume of $Q_j$ under the Lebesgue measure. Since the $Q_j$(s) are disjoint, the functions $\tilde{\phi}_{\gamma,j}$ have disjoint supports whose 
union makes up $\R^d$. 

Notice first that, 
\begin{align*}
    \tilde{\phi}_{\gamma,j}(x) &= \phi_\gamma \big( m(\xi_d(x) + q_j) \big ) \\
    &= \prod_{l=1}^d \sigma_{\gamma} \big ( m (\xi (x_l) + q_{jl}) \big ).
\end{align*}
Let $\nu \in \N^d$ be a multi-index with $|\nu|_1 = k$. We have, by Fa{\`a} di Bruno's formula,
\begin{align*}
    D^\nu \tilde{\phi}_{\gamma,j}(x) &=  \prod_{l=1}^d \frac{\partial^{\nu_l}}{\partial x^{\nu_l}_l} \sigma_\gamma \big ( m (\xi (x_l) + q_{jl}) \big ) \\
    &= \prod_{l=1}^d \sum_{\pi \in \Pi_{\nu_l}} \sigma^{(|\pi|)}_\gamma \big ( m (\xi(x_l) + q_{jl}) \big ) \prod_{B \in \pi} m \xi^{(|B|)}(x_l)
\end{align*}
where $\Pi_{\nu_l}$ is the set of partitions of the set $\{1,\dots,\nu_l\}$ with the convention that $\Pi_0 = \{\emptyset\}$. For example, 
\[\Pi_3 = \left \{ \big \{ \{1,2,3\} \big \}, \big \{ \{1\}, \{2\}, \{3\} \big \}, \big \{ \{1, 2\}, \{3\} \big \}, \big \{ \{1, 3\}, \{2\} \big \}, \big \{ \{3, 2\}, \{1\} \big \}  \right \}.\]
It follows, again by Fa{\`a} di Bruno's formula, that
\[D^\nu  \phi_{\gamma,j}(x) = \prod_{l=1}^d \sum_{\pi \in \Pi_{\nu_l}} \sigma^{(|\pi|)}_\gamma \big ( m (x_l + q_{jl}) \big ) \prod_{B \in \pi} m \]
hence, since $\xi \in C^\infty_b (\R)$, there exists a constant $H = H(\xi, k) > 0$ such that
\begin{equation}
    \label{eq:phitilde_phi_bound}
    \big | (D^\nu \tilde{\phi}_{\gamma,j})(x) \big | \leq H \big | (D^\nu \phi_{\gamma,j}) \big ( \xi_d (x) \big ) \big |.
\end{equation}
It then follows from \eqref{eq:bound_phi_j_der}  that
\begin{equation}
\label{eq:bound_phitilde_j_der}
\|D^\nu \tilde{\phi}_{j,\gamma}\|_\infty \leq \frac{H \Gamma m^{|\nu|_1}}{(1-\gamma)^{|\nu|_1}}, \qquad 0\leq |\nu|_1 \leq k.
\end{equation}
Furthermore, by the change of variables formula, we find
\begin{align}
    \label{eq:lp-gaussian_bump}
    \begin{split}
    \|\tilde{\phi}_{\gamma,j}\|^p_{L^p_{\rho_d}} &= \int_{E_j} |\tilde{\phi}_{\gamma,j}(x)|^p \rho_d (x) \: \mathsf{d}x \\
    &= \int_{E_j} \big | \phi_{\gamma,j} \big ( \xi_d(x) \big ) \big |^p \rho_d (x) \: \mathsf{d}x \\
    &= \int_{\xi_d(E_j)} |\phi_{\gamma,j}(x)|^p (\xi_d)_\sharp \rho_d (x) \: \mathsf{d}x \\
    &= \int_{Q_j} |\phi_{\gamma,j}(x)|^p \: \mathsf{d}x \\
    &\geq n^{-1} \gamma^d
    \end{split}
\end{align}
where the last line follows from \eqref{eq:lp-bump}.


\subsubsection{Proof of Theorem \ref{thm:sampling_width_gaussian}}

\begin{proof}
    The case $p = \infty$, follows precisely as Theorem~\ref{thm:sampling_width}. Suppose $p < \infty$ and define
    \[f = J \sum_{j=1}^{2n} \alpha_{j} \tilde{\phi}_{\gamma,j}\]
    for some $\alpha_j \in \{-1,1\}$ and constant $J$ such that $\|f\|_{W^{k,q}} \leq 1$. Due to \eqref{eq:bound_phitilde_j_der}, for $q = \infty$, the same calculation as in Theorem~\ref{thm:sampling_width} holds. When $q < \infty$, we have
    \begin{align*}
    \|f\|_{W^{k,q}_{\rho_d}}^q &= \sum_{0 \leq |\nu|_1 \leq k} \|D^\nu f \|_{L^q_{\rho_d}}^q \\
    &= J^q \sum_{0 \leq |\nu|_1 \leq k} \sum_{j=1}^{2n} \int_{E_j} |D^\nu \tilde{\phi}_{\gamma,j} (x) |^q \rho_d (x) \: \mathsf{d}x  \\
    &\leq J^q H^q \sum_{0 \leq |\nu|_1 \leq k} \sum_{j=1}^{2n} \int_{E_j} \big| D^\nu \phi_{\gamma,j}  \big ( \xi_d (x) \big ) \big |^q \rho_d (x) \: \mathsf{d}x  \\
    &= J^q H^q \sum_{0 \leq |\nu|_1 \leq k} \big\| D^\nu \sum_{j=1}^{2n} \phi_{\gamma,j} \big \|^q_q \\
    &\leq J^q H^q \sum_{0 \leq |\nu|_1 \leq k} \big\| D^\nu \sum_{j=1}^{2n} \phi_{\gamma,j} \big \|^q_\infty\\
    &\leq J^q (k+1) H^q \Gamma^q \left( 1+ \frac{d}{kp}\right)^{kq}  d^k  \Gamma^q n^{kq/d} 
    \end{align*}
    which follows by \eqref{eq:phitilde_phi_bound} and the change of variables formula similarly to \eqref{eq:lp-gaussian_bump}.
    In particular, we find that
    \[J =  (k+1)^{-1/q} H^{-1} \Gamma^{-1} \left (1 + \frac{d}{kp} \right )^{-k} d^{-k/q} n^{-k/d}\]
    implies $f \in K$. By a similar construction of the function $g$ as in Theorem~\ref{thm:sampling_width}, we find that, we need only to estimate
    \begin{align*}
        \|f - g\|_{L^p_{\rho_d}}^p &= 2^p J^p \sum_{\substack{j \in [2n], \\ j \notin l([n])}} \|\tilde{\phi}_{\gamma, j} \|^p_{L^p_{\rho_d}} \\
        & \geq 2^p J^p n^{-1} \gamma^d n \\
        &=  2^p J^p \gamma^d
    \end{align*}
    which follows by \eqref{eq:lp-gaussian_bump} and therefore the results follows as in Theorem~\ref{thm:sampling_width}.
\end{proof}

\subsubsection{Proof of Theorem \ref{thm:continous_width_gaussian}}

\begin{proof}
    Define $X_n \coloneqq \text{span } \{\tilde{\phi}_{\gamma,1}, \dots, \tilde{\phi}_{\gamma,n}\}$ and let $f \in X_n$. The result follows as in Theorem~\ref{thm:continous_width}; we need to only modify the estimates \eqref{eq:xn_lp_lowerbound} and \eqref{eq:xn_wkq_upperbound}. For \eqref{eq:xn_lp_lowerbound}, we have
    \begin{align}
    \label{eq:xn_lp_lowerbound_gaussian}
    \begin{split}
    \|f\|_{L^p_{\rho_d}}^p &= \sum_{l=1}^n |\alpha_l|^p \int_{E_j} |\tilde{\phi}_{\gamma,j}(x)|^p \rho_d (x) \: \mathsf{d}x \\
    &\geq n^{-1} \gamma^d \sum_{j=1}^n |\alpha_j|^p 
    \end{split}
\end{align}
which follows by \eqref{eq:lp-gaussian_bump}. For \eqref{eq:xn_wkq_upperbound}, we have
\begin{align}
    \label{eq:xn_wkq_upperbound_gaussian}
    \begin{split}
    \|f\|_{W^{k,q}_{\rho_d}} &= \left ( \sum_{0 \leq |\nu|_1  \leq k} \sum_{j=1}^n \int_{E_j} |\alpha_j|^q |D^\nu \tilde{\phi}_{\gamma,j} (x)|^q \rho_d(x) \: \mathsf{d}x  \right )^{1/q} \\
    &\leq \left ( \sum_{0 \leq |\nu|_1  \leq k} H^q \sum_{j=1}^n \int_{Q_j} |\alpha_j|^q |D^\nu \phi_{\gamma,j} (x)|^q \: \mathsf{d}x  \right )^{1/q} \\
    &\leq H \Gamma (k+1)^{1/q} \gamma^{-d/p} (1-\gamma)^{-k} d^{k/q} n^{k/d} \|f\|_{L^p_{\rho_d}} 
    \end{split}
\end{align}
which follows by \eqref{eq:bound_phitilde_j_der}, the change of variables formula, \eqref{eq:xn_lp_lowerbound_gaussian}, and \eqref{eq:xn_wkq_upperbound}. The result therefore follows as in Theorem~\ref{thm:continous_width}.
\end{proof}

\section{Lower Bounds in Infinite Dimensions}
\label{sec:lower_bounds_infd}

\subsubsection{Proof of Proposition \ref{prop:alpha-cubes}}
\label{sec:alpha-cubes}

\begin{proof}
By definition, a domain $D$ has non-empty interior. In particular, $D$ contains a $d$-dimensional cube. Furthermore, since $D$ is compact, it is also contained in a $d$-dimensional cube. Upon rescaling $D$, we will wlog assume that $[0,1]^d \subset D \subset [-N,N]^d$, for suitably chosen $N\in  \N$, in the following (a rescaling of $D$ only affects the implicit constants in our estimates, but not the decay rate $\alpha$).

\textbf{Case I: $L^p$-norm ($\cK= \cU(W^{s,p}(D))$):}
Assuming that $[0,1]^d\subset \Omega \subset [-N,N]^d$, we first show that $\cU(W^{s,p}(D))$ contains $\alpha$-hypercubes of arbitrary dimension $n$ for $\alpha = s/d + 1$. 

To see this, for given $n$, we consider $\phi_j$, $j=1,\dots, n$ the $n$ lowest-order $1$-periodic trigonometric (sine/cosine) basis functions, all of degree at most $m$, where $m$ is minimal such that $n \le m^d$. We note that for $1$-periodic functions, we have the norm equivalence,
\[
\Vert f \Vert_{W^{s,p}([0,1]^d)} 
\le \Vert f \Vert_{W^{s,p}(D)}
\le \Vert f \Vert_{W^{s,p}([-N,N]^d)}
\le 
C\Vert f \Vert_{W^{s,p}([0,1]^d)},
\]
where $C = C(N,d) \ge 0$ is a fixed constant, only depending on the dimension $d$ and on the domain $D$.
For coefficients $\alpha_1,\dots, \alpha_n \in \R$, and using the above norm equivalence for $f = \sum_{j=1}^n \alpha_j \phi_j$, it now follows from Bernstein's inequality that 
\begin{align*}
\left\Vert 
\sum_{j=1}^n \alpha_j \phi_j
\right\Vert_{W^{s,p}}
&\le C m^s \left\Vert 
\sum_{j=1}^n \alpha_j \phi_j
\right\Vert_{L^p},
\end{align*}
where $C = C(D, s)$ only depends on the underlying domain $D$ and $s$. Bounding the term on the right-hand side via the triangle inequality, and taking into account that $2^{-d}m^d \le n$, by choice of $m$, it follows that 
\begin{align*}
\left\Vert 
\sum_{j=1}^n \alpha_j \phi_j
\right\Vert_{W^{s,p}}
\le C m^s n \max_{j} |\alpha_j|
\le C n^{(s/d) + 1} \max_{j} |\alpha_j|,
\end{align*}
where $C = C(D, s, d)$ is independent of $n$. In particular, this implies that the set of functions,
\[
f = \frac{1}{C n^{(s/d)+1}} \sum_{j=1}^n  \alpha_j \phi_j, \quad |\alpha_j| \le 1,
\]
satisfies $\Vert f \Vert_{W^{s,p}} \le 1$, i.e. $f\in \cU(W^{s,p}(D))$. We also note that $\Vert \phi_j \Vert_{L^p(D)} \le C \Vert \phi_j \Vert_{L^p([0,1]^d)} \le M_0$ is uniformly bounded by a constant $M_0$ that depends on $\Omega$ and $d$, but that does not depend on $n$. Since $\phi_j$ are orthonormal in $L^2([0,1]^d)$, we can define dual elements $\phi^\ast_j$ by 
\[
\phi^\ast_j(f) = \int_{[0,1]^d} \phi_j(x) f(x) \, dx.
\]
And we note that $\phi^\ast_k(\phi_j)=\delta_{kj}$ for $j,k=1,\dots, n$ and $\Vert \phi^\ast_j \Vert_{(L^p(D))^\ast} = \Vert \phi_j \Vert_{L^q([0,1]^d)} \le M$ is uniformly bounded by a constant $M>0$, which depends on $D$ and $d$ but is independent of $n$.

The bi-orthogonal system $\{(\phi_j, \phi_j^\ast)\}$ thus satisfies all the assumptions for our definition of an $\alpha$-hypercube, except for the uniform bound $\Vert \phi_j \Vert_{L^p(D)}\le 1$. Since $\Vert \phi_j \Vert_{L^p(D)} \le M_0$, this is readily fixed by considering instead the bi-orthogonal system $\{(\tilde{\phi}_j, \tilde{\phi}_j^\ast)\} := \{(M_0^{-1}\phi_j, M_0\phi_j^\ast)\}$.

\textbf{Case II: supremum norm ($\cK= \cU(C^s(D))$):} Assuming without loss of generality that $[0,1]^d\subset D$, we next show that $\cU(C^s(D))$ contains $\alpha$-hypercubes of arbitrary dimension $n$ for $\alpha = s/d$.

Given $n\in \N$, choose $m\in \N$ minimal such that $n \le m^d$, and consider the partition of the unit cube $[0,1]^d = \bigcup_{j=1}^{m^d} Q_j$ and bump functions $\phi_j := \phi_{\gamma,j}|_{\gamma=1/2}$ constructed in \ref{subsec:partition_cube}. We recall that, by construction, each bump function has support $\supp(\phi_j) \subset Q_j$, and the $Q_j$ are (essentially) non-overlapping. Furthermore, we have
\[
\Vert D^\nu \phi_j \Vert_{L^{\infty}} \le C m^{|\nu|},
\]
with an absolute constant $C$, independent of $m$. We also note that, by choice of $m$, we have $c m^d \le n \le m^d$, for a fixed constant $c>0$ (e.g. $c = 2^{-d}$).

For coefficients $\alpha_1,\dots, \alpha_{n} \in \R$, and $f = J \sum_{j=1}^n \alpha_j \phi_j$, we have $\supp(f) \subset [0,1]^d$ and we obtain,
\begin{align*}
\Vert f \Vert_{C^s}
&= J \max_{|\nu|\le s} \max_{x\in [0,1]^d} | D^\nu f(x) |
\\
&= J \max_{|\nu|\le s} \max_{x\in [0,1]^d} \left\{ \max_{j=1,\dots, n} |\alpha_j| | D^\nu \phi_j(x) | \right\}
\\
&\le 
C J m^s \max_j |\alpha_j| .
\end{align*}
where $C = C(d,s,D)$ only depends on $d$, $s$ and $D$. In particular, using the fact that $m^s \le c n^{s/d}$ with $c = 2^{-d}$, it follows that there exists a constant $C = C(d,s,D)>0$, such that any $f$ of the form
\[
f = \frac{1}{C n^{s/d}} \sum_{j=1}^n \alpha_j \phi_j, \qquad |\alpha_j|\le 1,
\]
satisfies $\Vert f \Vert_{C^s} \le 1$, i.e. $f \in \cU(C^{s}(D))$. We can readily identify a bi-orthogonal system $\phi_k^\ast$ by noting that by construction of our bump functions $\{\phi_j\}$, for any $k=1,\dots, n$ there exists $x_k \in \Omega$, such that $\phi_j(x_k) = \delta_{jk}$. Therefore the functionals $\phi^\ast_j := \delta_{x_j}$ obtained by point-evaluation at $x_j$ define a bi-orthogonal system, with $\Vert \phi^\ast_j \Vert = 1$.

\end{proof}

\subsubsection{Proof of Lemma \ref{lem:embedding_cube}}
\label{sec:embedding_cube}



\begin{proof}
Fix $d\in \N$. By assumption, $\cK$ contains a hypercube of the form 
\[
\left\{
\frac{c}{d^\alpha} \sum_{j=1}^d y_j \phi_j \, : \, y_1,\dots, y_d \in [0,1]
\right\} \subset \cK,
\]
where $\phi_1,\dots, \phi_n$ possesses a dual basis $\phi_1^\ast, \dots \phi^\ast_n$ with $\Vert \phi_j^\ast \Vert_{\Omega^\ast} \le M$, and $M>0$ is independent of $d$.

For any $f \in C^k_0 ([0,1]^d)$, we will continue to denote by $f$ its zero-extension to all of $\R^d$. Note that this extension is in $C^k (\R^d)$ with the same norm. 
Define $\iota_d: C^k_0 ([0,1]^d) \to C^k(\cK)$, by
\begin{align}
\iota_d(f)(u) := f \big ( c^{-1}d^\alpha \phi^\ast_1(u), \dots, c^{-1}d^\alpha \phi^\ast_d(u) \big ).
\end{align}
Define $h_d : [0,1]^d \to \cK$ by
\[
h_d(y) \coloneqq \frac{c}{d^\alpha} \sum_{j=1}^d y_j \phi_j.
\]
Continuity of $h_d$ follows immediately since
\[\|h_d(y) - h_d(y')\|_{\Omega} \leq R_1 |y - y'|_1, \qquad \forall \: y, y' \in [0,1]^d\]
for some constant $R_1 > 0$. Note that, by construction,
\[
c^{-1}d^\alpha \phi^\ast_j \big ( h_d(y) \big ) = y_j, \quad j=1,\dots, d.
\]
By definition of $\iota_d f$ and $h_d$, we thus find, for any $y\in [0,1]^d$,
\begin{align*}
(\iota_d f) \big ( h_d(y) \big ) 
&= f \bigg( c^{-1}d^\alpha  \phi^\ast_1 \big ( h_d(y) \big ), \dots, c^{-1}d^\alpha  \phi^\ast_d \big ( h_d(y) \big ) \bigg )
= f(y).
\end{align*}
From the inclusion $h_d \big ([0,1]^d \big ) \subseteq \cK$, we obtain
\begin{align*}
\Vert \iota_d f \Vert_{C(\cK)}
&= \sup_{u\in \cK} |\iota_d f(u)|
\ge \sup_{y \in [0,1]^d} |\iota_d f \big ( h_d(y) \big )|
= \sup_{y\in [0,1]^d} |f(y)| = \Vert f \Vert_{C([0,1]^d)}
\end{align*}
which proves the asserted lower bound. We now establish the upper bound. To that end, fix $u\in \cK$, and let $w = \big( c^{-1}d^\alpha  \phi^\ast_{1}( u), \dots, c^{-1}d^\alpha \phi^\ast_{d}(u) \big )$. The $\ell$-th total derivative $D^\ell \iota_d f$ of $\iota_d f$, where $\ell \leq k$, is given by
\[
(D^\ell \iota_d f) (u)(v_1,\dots, v_\ell)
=
\sum_{j_1,\dots,j_\ell =1}^{d}
\frac{\partial^\ell f(w)}{\partial x_{j_1} \dots \partial x_{j_\ell}} 
\prod_{s=1}^\ell c^{-1}d^\alpha \phi^\ast_{j_s}(v_s)
\]
for any $v_1,\dots,v_\ell \in \Omega$. Hence,
\[\big | (D^\ell \iota_d f) (u)(v_1,\dots, v_\ell) \big | \leq c^{-l} M^l d^{\alpha l} \sum_{j_1,\dots,j_\ell =1}^{d}
\left | \frac{\partial^\ell f(w)}{\partial x_{j_1} \dots \partial x_{j_\ell}} \right |.\]
By definition, we have 
\[
\left|
\frac{\partial^\ell f(w)}{\partial x_{j_1} \dots \partial x_{j_\ell}}
\right|
\le 
\Vert f \Vert_{C^k(\R^d)} = \Vert f \Vert_{C^k([0,1]^d)}.
\]
Since the sum over $j_1,\dots, j_\ell$ has $d^{\ell}$ terms, there is a constant $R = R(k,\cK) > 1$ such that 
\[
\Vert (D^\ell \iota_d f) (u) \Vert_{\Omega^{\ell} \to \R}
\le d^{\alpha \ell} d^{\ell} M^\ell R_2^{\ell} \Vert f \Vert_{C^k([0,1]^d)}
\le
R d^{(\alpha+1)k} \Vert f \Vert_{C^k([0,1]^d)}.
\]
It follows that
\begin{align*}
    \|\iota_d f \|_{C^k (\cK)} &= \max_{1 \leq \ell \leq k} \sup_{u \in \cK} \max \big \{ |\iota_d f (u)|, \Vert (D^\ell \iota_d f) (u) \Vert_{\Omega^\ell \to \R} \big \} \\ 
    &\leq R d^{(\alpha + 1)k} \|f\|_{C^k ([0,1]^d)},
\end{align*}
as desired.
\end{proof}

\subsubsection{Proof of Theorem \ref{thm:ol-gaussian}}
\label{sec:ol-gaussian}

The proof of Theorem \ref{thm:ol-gaussian} is analogous to the proof of Theorem \ref{thm:ol-uniform}, except that Lemma \ref{lem:embedding_cube} is replaced by the following lemma:

\begin{lemma}
    \label{lemma:embed_ck_gaussian}
    Let $\cX$ be a Banach space with a bounded, bi-orthogonal system and let $\mu$ be a non-degenerate, Radon, Gaussian measure on $\cX$
    with trace-class covariance operator $\Gamma : \cX \to \cX$. Denote by $\lambda_1 \geq \lambda_2 \geq \dots$ the eigenvalues of $\Gamma$, ordered by their multiplicities and suppose that 
    there exists some $\alpha > 0$ such that $\sqrt{\lambda_j} \asymp j^{-\alpha}$ for all $j \in \N$. Then, for any $k,d \in \N$,
    there exists a 
    linear embedding $\iota_d : C^k (\R^d) \embeds C^k (\cX)$ and a constant $R = R(k, \cX) > 0$, such that, for any $f \in C^k(\R^d)$,
    \[\|\iota_d f\|_{C^k(\cX)} \leq C d^{k(\alpha + 2)} \|f\|_{C^k(\R^d)}\]
    and, for any $1 \leq p < \infty$,
    \[\|\iota_d f\|_{L^p_\mu (\cX)} = \|f\|_{L^p_{\rho_d} (\R^d)}.\]
\end{lemma}
\begin{proof}
    Define $\Phi^*_d : \cX \to \R^d$ by $\Phi^*_d = (\phi^*_1, \dots, \phi^*_d)$. By definition, $(\Phi^*_{d})_\sharp \mu = N(0, \Gamma_d)$ for some positive definite $\Gamma_d \in \R^{d \times d}$ whose eigenvalues satisfy the same bound as those of $\Gamma$. Therefore, $(\Gamma^{-1/2}_d \Phi^*_{d})_\sharp \mu = N(0, I_d)$. For any $f \in C^k (\R^d)$ define $\iota_d : C^k (\R^d) \to C^k (\cX)$ by $f \mapsto f \circ \Gamma^{-1/2}_d \Phi_d^*$. By the change of variables formula,
    \begin{align*}
        \|\iota_d f \|_{L^p_\mu (\cX)}^p &= \int_{\cX} |f \big ( \Gamma^{-1/2}_d \Phi_d^* (u) \big) |^p d \mu (u) \\
        &= \int_{\R^d} |f(x)|^p  d (\Gamma^{-1/2}_d \Phi^*_{d})_\sharp \mu (x) \\
        &= \|f\|_{L^p_{\rho_d}(\R^d)}^p
    \end{align*}
    which establishes the second result. To prove the claimed upper-bound, note that the 
    $l$-th total derivative of $\iota_d f$ is the $l$-linear mapping
    \[(D^l \iota_d f)(u)(v_1,\dots,v_l) = \sum_{j_1,\dots,j_\ell =1}^{d} \frac{\partial^\ell f(y)}{\partial x_{j_1} \dots \partial x_{j_\ell}}  \prod_{s=1}^l \langle (\Gamma^{-1/2}_d)_{j_s}, \Phi_d^* (v_s) \rangle\]
    for any $u \in \cX$ and $v_1,\dots,v_l \in \cX$ where $y = \Gamma^{-1/2}_d \Phi_d^* (u)$. For any $\|v_1\|_{\cX} \leq 1, \dots, \|v_l\|_{\cX} \leq 1$, we have
    \begin{align*}
        \prod_{s=1}^l \langle (\Gamma^{-1/2}_d)_{j_s}, \Phi_d^* (v_s) \rangle &\leq \prod_{s=1}^l |(\Gamma^{-1/2}_d)_{j_s}|_2 |\Phi_d^* (v_s)|_2 \\
        &\leq d^{l/2} M^l \prod_{s=1}^l |(\Gamma^{-1/2}_d)_{j_s}|_1 \\
        &\leq d^{l/2} M^l |\Gamma^{-1/2}_d|_\infty^l \\
        &\leq d^{l} M^l |\Gamma^{-1/2}_d|_2^l \\
        &\leq M^l d^{l + \alpha l} 
    \end{align*}
    where we have absorbed the constant from the eigenvalue bounds of $\Gamma_d$ into M. It follows that
    \[\|D^l \iota_d f\|_{C(\cX)} \leq M^l d^{l(\alpha + 2)} \|f\|_{C^k (\R^d)}\]
    and therefore
    \[\|\iota_d f\|_{C^k (\cX)} \leq M^k d^{k(\alpha + 2)} \|f\|_{C^k (\R^d)}\]
    as desired.
\end{proof}

\begin{theorem}
\label{thm:lp-gauss}
    Assume the setting of Lemma~\ref{lemma:embed_ck_gaussian}. Let $\U := \cU(C^k(\cX))$. Then there exists a constant $C = C(k,\alpha) > 0$
    such that
    \begin{align}
    d_n(\U)_{L^p_\mu(\cX)}
    \ge
    C n^{-\frac{1}{p}} \log (n) ^{-k(\alpha + 3)}.
    \end{align}
\end{theorem}
\begin{proof}
    A slightly sub-optimal choice of $\gamma$ in the proof of Theorem~\ref{thm:continous_width_gaussian}, gives the bound
    \[d_n(\cU(C^k(\R^d)))_{L^p_{\rho_d} (\R^d)} \geq C d^{-k} n^{-\frac{k}{d} - \frac{1}{p}}.\]
    Combining with Lemma~\ref{lemma:embed_ck_gaussian}, and arguing as in the proof of Theorem \ref{thm:ol-uniform}, we obtain
    \begin{align}
    d_n(\U)_{L^p_\mu (\cX)}
    \ge
    C d^{-k(\alpha + 3)} n^{-\frac{k}{d} - \frac{1}{p}},
    \end{align}
    with $C>0$ independent of $d$. Optimizing over $d$ yields the choice $d \approx \log (n) / (\alpha + 3)$ and the result follows upon substitution into the 
    above equation.
\end{proof}

\section{Data-efficient Upper Bounds}
\label{sec:data-efficient}

The proof of Theorem \ref{thm:fno} in the main text rests on the quantitative bounds from Proposition \ref{prop:erm} and Proposition \ref{prop:fno-covering}, whose derivation is finally detailed in the following two subsections.

\subsubsection{Proof of Proposition \ref{prop:erm}}
\label{pf:erm}
The proof of Proposition \ref{prop:erm} follows a well-known strategy in statistical learning theory; we now proceed to the proof, but relegate several technical intermediate results, summarized as lemmas, to the appendices as indicated below.
\begin{proof}{(Proof of Proposition \ref{prop:erm})}
Let $\cG \in \cU(\cA^\gamma)$ be given.
And let $u_1,\dots, u_n \in L^2(D)$ be arbitrary input functions for the moment -- a good choice of these input functions will be determined later. By assumption $\Psi_\cG$ is an empirical risk minimizer, i.e. it minimizes 
\[
\hL(\Psi_\cG;\cG) = \min_{\Psi\in \Sigma_{m}'} \hL(\Psi;\cG).
\]
Our aim is to estimate $\Vert \Psi_\cG - \cG \Vert_{L^2_\mu}$, or equivalently, to estimate the population risk,
\[
\cL(\Psi_\cG;\cG) = \E_{u\sim \mu}\left[
\Vert \Psi_\cG(u) - \cG(u) \Vert^2_{L^2(D)}
\right]
= 
\Vert \Psi_\cG - \cG \Vert_{L^2_\mu}^2.
\]
If $u_1,\dots, u_n \sim \mu$ are iid randomly drawn with law $\mu$, we can think of $\hL$ as a Monte-Carlo estimate of $\cL$.

Following a well-known strategy from statistical learning theory, our first step in estimating the population risk $\cL(\Psi_\cG;\cG)$ is to split this risk into an approximation error contribution and an estimation error contribution. Upon bounding the approximation error based on the assumption that $\cG \in \cU(\cA^\gamma)$, this results in the following lemma, whose detailed proof can be found in Appendix \ref{app:erm-decomp}:
\begin{lemma}
\label{lem:erm-decomp}
Let $\cG \in \cU(\cA^\gamma)$ and $m\in \N$ be given. Let $\Psi_\cG \in \Sigma_m'$ denote the empirical risk minimizer. Then, 
\begin{align}
\label{eq:cLupper}
\cL(\Psi_\cG; \cG) 
\le
m^{-2\gamma} + 2\sup_{\cG \in \cU(\cA^\gamma), \Psi \in \Sigma_m'}[\cL(\Psi;\cG) - \hL(\Psi;\cG)].
\end{align}
\end{lemma}
In \eqref{eq:cLupper}, the first term on the right bounds the approximation error, while the second term bounds the estimation error, which arises due to the finite number of samples.

To bound the estimation error contribution, we will follow some ideas of \cite[Proof of Theorem C${}^*$]{CS1}. To this end, we first note that for any fixed pair $(\Psi, \cG) \in \Sigma_m' \times \cU(\cA^\gamma)$, we have a uniform bound
\begin{align}
\label{eq:PGbd}
\begin{aligned}
\sup_{u\in K}
\Vert \Psi(u) - \cG(u) \Vert_{L^2(D)}
&= 
\Vert \Psi - \cG \Vert_{C(\cK;L^2(D))}
\\
&\le 
\Vert \Psi \Vert_{C(\cK;L^2(D))} + \Vert \cG \Vert_{C(\cK;L^2(D))}
\\
&\le 3.
\end{aligned}
\end{align}
For iid samples $u_1,\dots, u_n \sim \mu$, we now define iid random variables  
\[
Y_j := \Vert \Psi(u_j) - \cG(u_j) \Vert_{L^2(D)}^2.
\]
Note that the random variables satisfy $0 \le Y_j \le 9$, and in particular, the $Y_j$ are non-negative. Therefore, by Bernstein's inequality \cite[Prop. 2.14, in particular (2.23)]{wainwright2019high}, we obtain that for any $\beta > 0$:
\begin{align}
\label{eq:bernstein}
\Prob\left[
\cL(\Psi;\cG) - \hL(\Psi;\cG)
\ge \beta
\right]
&= 
\Prob\left[
\E[Y_j] - \frac 1n \sum_{j=1}^n Y_j
\ge \beta
\right]
\\
&\le 
\exp\left(
-
\frac{\beta^2 n}{\frac{2}{n} \sum_{j=1}^n \E[Y_j^2]}
\right),
\end{align}
where, by \eqref{eq:PGbd}, we can estimate
\[
\frac1n \sum_{j=1}^n\E[Y_j^2] \le 9 \E\Vert \Psi - \cG \Vert^2 = 9 \cL(\Psi;\cG).
\]
Replacing $\beta$ by the product $\alpha \, (\cL(\Psi;\cG) + \rho )$ for $\alpha,\rho > 0$ (we note that the quantity $\cL(\Psi;\cG)$ is independent of the samples and \emph{not random}), this implies that 
\begin{align*}
\Prob\left[
\frac{  
\cL(\Psi;\cG) - \hL(\Psi;\cG)
}
{
\cL(\Psi;\cG) + \rho
}
\ge \alpha
\right]
&\le 
\exp\left(
-
\frac{\alpha^2 n (\cL(\Psi;\cG)  + \rho)^2}{18 \cL(\Psi;\cG) }
\right).
\end{align*}
Using the trivial inequality $(\cL(\Psi;\cG) + \rho)^2 \ge 2 \rho \cL(\Psi;\cG)$, this in turn implies
\begin{align}
\Prob\left[
\frac{  
\cL(\Psi;\cG) - \hL(\Psi;\cG)
}
{
\cL(\Psi;\cG) + \rho
}
\ge \alpha
\right]
&\le 
\exp\left(
-
\frac{\alpha^2 n \rho}{9}
\right),
\end{align}
for any fixed pair $(\Psi,\cG) \in \Sigma_m'\times \cU(\cA^\gamma)$. For future reference, we note this bound in the following lemma:

\begin{lemma}
Let $(\Psi,\cG)\in \Sigma_m'\times \cU(\cA^\gamma)$, and $\cL(\Psi,\cG) = \E_{u\sim \mu} \Vert \Psi(u) - \cG(u) \Vert^2_{L^2}$ the population risk, and $\hL(\Psi,\cG) = \frac1n \sum_{j=1}^n \Vert \Psi(u_j) - \cG(u_j) \Vert^2_{L^2}$ the empirical risk. If the samples $u_1,\dots, u_n \sim \mu$ are chosen iid, then for any $\rho, \alpha > 0$ and $n\in \N$, we have
\begin{align}
\label{eq:bernstein-relative}
\Prob\left[
\frac{  
\cL(\Psi;\cG) - \hL(\Psi;\cG)
}
{
\cL(\Psi;\cG) + \rho
}
\ge \alpha
\right]
&\le 
\exp\left(
-
\frac{\alpha^2 n \rho}{9}
\right).
\end{align}
\end{lemma}

A similar bound can be derived for the supremum over all pairs $(\Psi,\cG) \in \Sigma_m' \times \cU(\cA^\gamma)$, via simple a union bound argument. This is summarized in the next lemma, with proof detailed in Appendix \ref{app:estprob}. The bound below also makes a specific choice for $\rho \propto \alpha^{-1} \epsilon$:
\begin{lemma}
\label{lem:estprob}
Let $\epsilon > 0$, $\alpha \in (0,1]$. Denote $\cH = \Sigma_{m}'\times \cU(\cA^\gamma)$. Then 
\begin{align}
\Prob\left[
\sup_{(\Psi,\cG)\in \cH}
\frac{
\cL(\Psi;\cG) - \hL(\Psi;\cG)
}{
\cL(\Psi;\cG) + 72 \alpha^{-1}\epsilon
}
\ge 
\alpha
\right]
\le
N_{\cH,\epsilon} \exp\left( -2\alpha n\epsilon \right).
\end{align}
 Here $N_{\cH,\epsilon}$ can be thought of as bounding the covering number of $\cH$, and for  $m \ge (2/\epsilon)^{1/\gamma}$, we can estimate
\begin{align}
N_{\cH,\epsilon} \le \cN\left(\Sigma_{m}, \epsilon\right)^2,
\end{align}
where $\cN\left(\Sigma_{m}, \epsilon\right)$ denotes the $\epsilon$-covering number of $\Sigma_m$ with respect to the $C(\cK;L^2(D))$-norm.
\end{lemma}

The last lemma implies the following corollary, also proved in Appendix \ref{app:estprob}:
\begin{corollary}
\label{cor:estprob}
Fix $\delta \in (0,1]$. Let $\epsilon > 0$ be given, and let $m \ge \left(2/\epsilon\right)^{1/\gamma}$ be integer. If 
\[
n \ge \epsilon^{-1} \log(\delta^{-1}\cN(\Sigma_m,\epsilon)^2),
\]
then with probability at least $1-\delta$ in the random samples $u_1,\dots, u_n \sim \mu$, we have
\begin{align}
\label{eq:estprob}
\cL(\Psi;\cG) \le 2\hL(\Psi;\cG) + 144 \epsilon, 
\end{align}
uniformly for all $(\Psi,\cG) \in \Sigma_m'\times \cU(\cA^\gamma)$.
\end{corollary}

Given Corollary \ref{cor:estprob}, the claim of Proposition \ref{prop:erm} now follows by setting $\Psi = \Psi_\cG$ in \eqref{eq:estprob}, and noting that $\Psi_\cG$ minimizes the first term on the right over all possible choices $\Psi\in \Sigma_m$, and hence:
\begin{align*}
\cL(\Psi_\cG; \cG) 
&\le 2 \hL(\Psi_\cG; \cG) + 144 \epsilon
\\
&\le 2 \inf_{\Psi\in \Sigma_m} \hL(\Psi; \cG) + 144 \epsilon
\\
&\le 2 \inf_{\Psi\in \Sigma_m} \Vert \Psi - \cG \Vert_{C(\cK;L^2(D))}^2 + 144\epsilon
\\
&\le 2 m^{-2\gamma} + 144 \epsilon.
\end{align*}
By assumption, we have $m^{-2\gamma} \le (\epsilon/2)^2 \le \epsilon / 2$, and hence, 
\[
\cL(\Psi_\cG; \cG) 
\le 
145 \epsilon.
\]
This concludes our proof of Proposition \ref{prop:erm}.
\end{proof}

\subsubsection{Proof of Proposition \ref{prop:fno-covering}}
\label{pf:fno-covering}

Our proof of Proposition \ref{prop:fno-covering} relies on several lemmas. We will provide a proof of Proposition \ref{prop:fno-covering} and Corollary \ref{cor:sigmam-covering} here, and refer the reader to Appendix \ref{app:entropy} for the detailed proofs of the lemmas.

\begin{proof}{(Poof of Proposition \ref{prop:fno-covering})}

Given a FNO architecture $\Psi(\slot;\theta)$ with fixed hyper-parameters $\dc$, $\kappa$ and $L$, our goal is to bound the covering number of the set $\FNO\subset C(\cK;L^2(D))$ of operators that can be represented by this architecture. To this end, we note that this set is the image under the mapping $[-B,B]^{d_\theta} \to C(\cK;L^2(D))$, $\theta \mapsto \Psi(\slot; \theta)$, from parameters $\theta$ to corresponding operator $\Psi(\slot;\theta)$. This simple observation is relevant in view of the following elementary lemma (cp. Appendix \ref{app:entropy1} for a proof):
\begin{lemma}
\label{lem:lip-covering}
If $F: [-B,B]^d \to Y$ is a Lipschitz continuous mapping into a Banach space $Y$, then 
\[
\cN(F([-B,B]^d);\epsilon) \le \cN([-B,B]^d;\epsilon/\Lip(F)).
\]
\end{lemma}

In parameter space, estimates for the covering number of a cube $[-B,B]^d \subset \R^d$ are well-known. In particular, the covering number of the hypercube $[-B,B]^d$ with respect to the $\ell^\infty$-norm, $|\slot|_\infty$, is bounded by,
\begin{align}
\label{eq:cube-covering}
\cN([-B,B]^d, \epsilon)_{|\slot|_\infty} \le \left(\frac{2B}{\epsilon} \right)^d.
\end{align}
This follows by covering the cube $[-B,B]^d$ by sub-cubes with side-length $\epsilon$ in each direction. Since $2B/\epsilon$ intervals of length $\epsilon$ cover $[-B,B]$, we require at most $(2B/\epsilon)^d$ sub-cubes to cover $[-B,B]^d$.

In order to apply Lemma \ref{lem:lip-covering} and \eqref{eq:cube-covering} to obtain an estimate of the covering number of the space of FNOs represented by each $\Psi(\slot;\theta)$ introduced above, it remains to bound the Lipschitz constant of the mapping $\theta \mapsto \Psi(\slot, \theta)$. This is the subject of the following lemma:
\begin{lemma}
\label{lem:fno-Lip}
Let $\cK\subset L^2(D)$ be a compact set. Let 
\[
M := 
\max_{u\in \cK} 
\Vert u \Vert_{L^2(D)} < \infty.
\]
The Lipschitz constant of the mapping $[-B,B]^{d_\theta} \to C(\cK;L^2(D))$, $\theta \mapsto \Psi(\slot;\theta)$ is bounded by
\[
\Lip\left(\theta \mapsto \Psi(\slot;\theta)\right)
\le 
(L+2) (2\dc B)^{L+2} \left( M + (2\kappa)^{d/2}\right).
\]
\end{lemma}

For a proof of this lemma, we refer to Appendix \ref{app:fno-Lip} (cp. Proposition \ref{prop:fno-Lip}).
Combining Lemma \ref{lem:lip-covering} with \eqref{eq:cube-covering} and Lemma \ref{lem:fno-Lip}, it now follows that 
\begin{align*}
\cN(\FNO, \epsilon)_{C(\cK;L^2(D))}
&\le 
\cN\left([-B,B]^{d_{\theta}},\frac{\epsilon}{\Lip(\theta \mapsto \Psi(\slot;\theta))}\right) 
\\
&\le 
\left( 
\frac{2B \Lip(\theta \mapsto \Psi(\slot;\theta))}{\epsilon}
\right)^{d_\theta}.
\end{align*}
We also recall that, by \eqref{eq:dtheta},
\[
d_\theta \le 5 (2\kappa)^d L \dc^2.
\]
Taking logarithms in the above bound on $\cN(\FNO, \epsilon)$, it now follows that there exists a constant $C = C(d,M)>0$, depending only on the dimension of the function domain $\Omega$, and the norm bound $M$ over the compact set $\cK$, such that
\[
\log \cN(\FNO,\epsilon)_{C(\cK;L^2(D))}
\le 
C \kappa^d L^2 \dc^2 \log\left(\frac{\dc \kappa L B}{\epsilon}\right).
\]
This completes our proof of Proposition \ref{prop:fno-covering}.
\end{proof}

We next deduce Corollary \ref{cor:sigmam-covering} from Proposition \ref{prop:fno-covering}.

\begin{proof}{(Proof of Corollary \ref{cor:sigmam-covering})}
For $\ell=1,\dots, m$, let $\Sigma_{m}^{(\ell)}$ denote the set of FNOs with $\dc\le m$, $\kappa^d \le m$, parameter bound $\Vert \theta \Vert_\infty \le B := \exp(m)$ and depth $\ell$. One readily verifies that
\begin{align}
\label{eq:sigmaml}
\Sigma_m = \Sigma_m^{(1)} \cup \dots \cup \Sigma_m^{(m)},
\end{align}
and any operator $\Psi \in \Sigma_m^{(\ell)}$ can be represented by a specific choice of the parameters in the FNO architecture $\Psi_\ell(\slot;\theta_\ell)$ of depth $\ell$, with $\dc=m$, $\kappa^d = m$, and with $\theta_\ell \in \R^{d_{\theta_\ell}}$ satisfying parameter bound $\Vert \theta_\ell \Vert_\infty \le B$.

By Proposition \ref{prop:fno-covering}, we have
\begin{align*}
\cN(\Sigma_m^{(\ell)},\epsilon)
&\lesssim
\kappa^d L^2 \dc^2 \log\left(
\frac{\dc \kappa L B}{\epsilon}
\right)
\\
&\lesssim
m^{5} \log\left(
\frac{m^3 e^m}{\epsilon}
\right)
\\
&\lesssim
m^{6} \log(m\epsilon^{-1}),
\end{align*}
with an implied constant that only depends on $d$ and $M = \max_{u\in \cK} \Vert u \Vert_{L^2(D)}$.

From \eqref{eq:sigmaml}, it thus follows that 
\[
\cN(\Sigma_m,\epsilon)
\le 
\sum_{\ell=1}^m \cN(\Sigma_m^{(\ell)},\epsilon)
\le 
Cm^{7} \log(m\epsilon^{-1}),
\]
with a constant $C = C(d,\cK)>0$. This is the claimed estimate on the metric entropy of $\Sigma_m$, and concludes our proof of Proposition \ref{prop:fno-covering}.
\end{proof}

\subsection{Estimates for the empirical risk minimizer}

The proof of Proposition \ref{prop:erm} uses the following lemma
\begin{lemma}
\label{lem:coveringW1}
Let $\epsilon > 0$ be given. The covering number of $\cU(\cA^\gamma)$, with respect to the $\Vert \slot \Vert_{L^2_\mu}$ norm, is bounded by
\begin{align}
\label{eq:coveringW1}
\cN(\cU(\cA^\gamma), 2\epsilon) \le \cN(\Sigma_m,\epsilon),
\end{align}
for $m \ge \epsilon^{-1/\gamma}$.
\end{lemma}

\begin{proof}{(Proof of Lemma \ref{lem:coveringW1})}
Choose $m$ as in the statement of the lemma. Let $N = \cN(\Sigma_m,\epsilon/2)$ be the $\epsilon$-covering number of $\Sigma_m$, and let $\Psi_1,\dots, \Psi_N$ be the centers of such a covering. By definition of $\cU(\cA^\gamma)$, for any $\cG\in \cU(\cA^\gamma)$, we have
\[
\inf_{\Psi\in \Sigma_m} \Vert \cG - \Psi \Vert_{C(\cK;L^2(D))} \le m^{-\gamma} \le \epsilon.
\]
By compactness of $\Sigma_m$ (cp. Remark \ref{rem:compact}), the infimum on the left is actually attained and thus, there exists $\Psi^\star \in \Sigma_m$, such that 
\[
\Vert \cG - \Psi^\star \Vert_{C(\cK;L^2(D))} \le \epsilon.
\]
Since the $\Psi_1,\dots, \Psi_N$ determine an open $\epsilon$-covering of $\Sigma_m$, it follows that 
\begin{align*}
\min_{j=1,\dots, N} \Vert \cG - \Psi_j \Vert_{C(\cK;L^2(D))} 
&\le 
\Vert \cG - \Psi^\star \Vert_{C(\cK;L^2(D))} + \min_{j=1,\dots, N} \Vert \Psi^\star - \Psi_j \Vert_{C(\cK;L^2(D))}
\\
&< \epsilon + \epsilon = 2\epsilon.
\end{align*}
Hence, the $N = \cN(\Sigma_m,\epsilon)$ balls of radius $\epsilon$ with centers $\Psi_1,\dots, \Psi_N$ cover $\cU(\cA^\gamma)$, implying that $\cN(\cU(\cA^\gamma),2\epsilon) \le \cN(\Sigma_m, \epsilon)$. 
\end{proof}

\subsection{Proof of Lemma \ref{lem:erm-decomp}}
\label{app:erm-decomp}

\begin{proof}{(Proof of Lemma \ref{lem:erm-decomp})}
We first observe that for any $\cG \in \cU(\cA^\gamma)$, the empirical risk minimizer $\Psi_\cG$ in \eqref{eq:erm} satisfies
\begin{align*}
\cL(\Psi_\cG; \cG) 
&= 
\hL(\Psi_\cG; \cG) + [\cL(\Psi_\cG;\cG) - \hL(\Psi_\cG;\cG)]
\\
&\le
\hL(\Psi_\cG; \cG) + \sup_{\cG, \Psi} 
\left|\cL(\Psi;\cG) - \hL(\Psi;\cG)\right|,
\end{align*}
where the supremum is taken over all $\cG\in \cU(\cA^\gamma)$, and $\Psi\in \Sigma_m'$.
Since $\Psi_\cG$ is a minimizer of $\hL$, we have $\hL(\Psi_\cG;\cG) \le \hL(\Psi';\cG)$ for any $\Psi' \in \Sigma_m'$. Fixing arbitrary $\Psi'$ for the moment, it follows that
\[
\hL(\Psi_\cG; \cG)
\le
\hL(\Psi';\cG) 
\le
\cL(\Psi';\cG) + \sup_{\cG, \Psi}
\left|\cL(\Psi;\cG) - \hL(\Psi;\cG)\right|, 
\]
and hence
\[
\cL(\Psi_\cG; \cG) 
\le
\cL(\Psi';\cG) + 
2\sup_{\cG, \Psi}
\left|\cL(\Psi;\cG) - \hL(\Psi;\cG)\right|.
\]
Since $\Psi'$ was arbitrary, we can finally take the infimum over all $\Psi'\in \Sigma_m'$ on the right, to obtain
\begin{align*}
\cL(\Psi_\cG; \cG) 
\le
\inf_{\Psi' \in \Sigma_m'} \cL(\Psi';\cG) + 
2\sup_{\cG, \Psi}
\left|\cL(\Psi;\cG) - \hL(\Psi;\cG)\right|.
\end{align*}
The first term represents the best-approximation of $\cG$ with respect to the (squared) $L^2_\mu$-norm. The second term is the generalization error. Since $\cG\in \cU(\cA^\gamma)$ belongs to the approximation space, we have
\[
\inf_{\Psi'\in \Sigma_m'} \cL(\Psi';\cG)
\equiv 
\inf_{\Psi'\in \Sigma_m'} 
\left\Vert \cG - \Psi' \right\Vert_{L^2_\mu}^2
\le m^{-2\gamma}.
\]
In particular, this implies that
\begin{align}
\label{eq:cLupper1}
\cL(\Psi_\cG; \cG) 
\le
m^{-2\gamma} + 2\sup_{\cG \in \cU(\cA^\gamma), \Psi \in \Sigma_m'}[\cL(\Psi;\cG) - \hL(\Psi;\cG)],
\end{align}
where we recalled that the supremum on the right is over all $\cG \in \cU(\cA^\gamma)$ and $\Psi\in \Sigma_m'$.
\end{proof}

\subsection{Proof of Lemma \ref{lem:estprob}}
\label{app:estprob}

Our derivation of the upper bound on the esimation error will require the following simple result:
\begin{lemma}
\label{lem:LLip}
For any $\Psi, \Psi' \in \Sigma_m'$ and $\cG, \cG' \in \cU(\cA^\gamma)$, and any choice of samples $u_1,\dots, u_n$, we have
\[
|\hL(\Psi;\cG) - \hL(\Psi',\cG')| \le 6 \left(\Vert \Psi - \Psi' \Vert_{C(\cK;L^2(D))} + \Vert \cG - \cG' \Vert_{C(\cK;L^2(D))}\right),
\]
and 
\[
|\cL(\Psi;\cG) - \cL(\Psi',\cG')| \le 6 \left(\Vert \Psi - \Psi' \Vert_{C(\cK;L^2(D))} + \Vert \cG - \cG' \Vert_{C(\cK;L^2(D))}\right).
\]
\end{lemma}
\begin{proof}
We first note that $|a^2 - b^2| \le 2\max(a,b) |a-b|$. By definition, we have for any $(\Psi,\cG) \in \Sigma_m' \times \cU(\cA^\gamma)$, the bounds
\[
\Vert \Psi(u_j) - \cG(u_j) \Vert_{L^2(D)}
\le \Vert \Psi \Vert_{C(\cK;L^2(D))} + \Vert \cG \Vert_{C(\cK;L^2(D))} \le 3.
\]
We now apply the above bounds, for each index $j$ individually, with 
\[
a = \Vert \Psi(u_j) - \cG(u_j) \Vert_{L^2(D)},
\quad
b = \Vert \Psi'(u_j) - \cG'(u_j) \Vert_{L^2(D)},
\]
to obtain
\begin{align*}
|\hL(\Psi;\cG) - \hL(\Psi',\cG')|
&\le 
\frac{1}{n} \sum_{j=1}^n \left|
\Vert \Psi(u_j) - \cG(u_j) \Vert^2_{L^2(D)} - \Vert \Psi'(u_j) - \cG'(u_j) \Vert^2_{L^2(D)}
\right|
\\
&=
\frac{6}{n} \sum_{j=1}^n 
\left|
\Vert \Psi(u_j) - \cG(u_j) \Vert_{L^2(D)} - \Vert \Psi'(u_j) - \cG'(u_j) \Vert_{L^2(D)}
\right|.
\end{align*}
The claimed bound on $\hL$ then follows from the triangle inequality,
\begin{align*}
\Big|
\Vert \Psi(u_j) - 
&\cG(u_j) \Vert_{L^2(D)} - \Vert \Psi'(u_j) - \cG'(u_j) \Vert_{L^2(D)}
\Big|
\\
&\le 
\Vert \left\{\Psi(u_j) - \cG(u_j)\right\} - \left\{\Psi'(u_j) - \cG'(u_j) \right\} \Vert_{L^2(D)}
\\
&\le 
\Vert \Psi(u_j) -\Psi'(u_j) \Vert_{L^2(D)} + \Vert \cG(u_j) -\cG'(u_j) \Vert_{L^2(D)} 
\\
&\le 
\Vert \Psi -\Psi' \Vert_{C(\cK;L^2(D))} + \Vert \cG -\cG' \Vert_{C(\cK;L^2(D))}.
\end{align*}
The bound on $\cL$ follows from the simple observation that $\cL(\Psi;\cG) = \E \hL(\Psi;\cG)$ is the expectation over the random samples, and hence
\begin{align*}
|\cL(\Psi;\cG) - \cL(\Psi';\cG')|
&= 
\left|
\E \hL(\Psi;\cG) - \E\hL(\Psi';\cG')
\right|
\\
&\le
\E \left| \hL(\Psi;\cG)- \hL(\Psi';\cG') \right|
\\
&\le 
6\left(\Vert \Psi -\Psi' \Vert_{C(\cK;L^2(D))} + \Vert \cG -\cG' \Vert_{C(\cK;L^2(D))}\right).
\end{align*}
\end{proof}

\begin{lemma}
\label{lem:flip}
Let $\epsilon > 0$, and define
\[
f(a,b) := \frac{(a-b)_+}{a+\epsilon}, \quad \forall a,b\ge 0.
\]
Then $f$ is globally Lipschitz continuous, and 
\[
\sup_{a,b} |\partial_a f(a,b)|\le \frac{1}{\epsilon},
 \quad 
 \sup_{a,b}|\partial_b f(a,b)| \le \frac{1}{\epsilon},
\]
almost everywhere.
\end{lemma}

\begin{proof}
This is a simple calculation. We provide the details for completeness. Clearly, $f$ is identically zero on the set $\{ 0 \le a < b\}$. On the other hand, on $\{0 \le  b\le a\}$, we have
\begin{align*}
\left|\partial_a f\right|
&= 
\frac{1}{a+\epsilon}\frac{b+\epsilon}{a+\epsilon} \le \frac{1}{\epsilon}, \quad
\left|\partial_b f\right| 
= \left| \frac{-1}{a+\epsilon}\right| \le \frac{1}{\epsilon}.
\end{align*}
\end{proof}

\begin{proof}{(Proof of Lemma \ref{lem:estprob})}
We recall that the empirical risk 
\[
\hL(\Psi;\cG)
= 
\frac1n \sum_{j=1}^n \Vert \cG(u_j) - \Psi(u_j) \Vert^2_{L^2(D)},
\]
depends on the choice of $u_1,\dots, u_n$. We choose them iid $u_1,\dots, u_n\sim \mu$ drawn from the probability measure $\mu$.

Let $\Psi_1,\dots, \Psi_N$ be a $\epsilon$-cover of $\Sigma_m'\subset \Sigma_m$ with $N = \cN(\Sigma_m',\epsilon)$. Let $\cG_1,\dots, \cG_M$ be a $2\epsilon$-cover of $\cU(\cA^\gamma)$ with $M = \cN(\cU(\cA^\gamma),2\epsilon)$. Then, we note that for any pair $(\Psi,\cG)\in \Sigma_m' \times \cU(\cA^\gamma)$, there exists $k\in [N]$ and $\ell\in [M]$, such that
\begin{align}
\label{eq:cov}
\Vert \Psi(u) - \Psi_k(u) \Vert_{C(\cK;L^2(D))} \le \epsilon,
\quad \Vert \cG(u) - \cG_\ell(u) \Vert_{C(\cK;L^2(D))} \le 2\epsilon.
\end{align}
Let $f(a,b)$ be the function introduced in Lemma \ref{lem:flip}, so that
\begin{align*}
\frac{\cL(\Psi;\cG) - \hL(\Psi;\cG)}{\cL(\Psi;\cG)+s}
&\le f(\cL(\Psi;\cG),\hL(\Psi;\cG))
\\
&\le 
f(\cL(\Psi_k;\cG_\ell),\hL(\Psi_k;\cG_\ell))
\\
&\qquad 
+ \Vert \partial_a f \Vert_{L^\infty} \left| \cL(\Psi;\cG) - \cL(\Psi_k;\cG_\ell) \right|
\\
&\qquad  
+ \Vert \partial_b f \Vert_{L^\infty}  \left| \hL(\Psi;\cG) - \hL(\Psi_k;\cG_\ell) \right|.
\end{align*}
By the estimate of Lemma \ref{lem:LLip} and \ref{lem:flip}, and by making a specific choice of $k,\ell$ depending on $(\Psi,\cG)$, we can estimate each of the second term by, 
\[
\Vert \partial_a f \Vert_{L^\infty}  \left| \cL(\Psi;\cG) - \cL(\Psi_k;\cG_\ell) \right|
\le 
\frac{6}{s} \left(\Vert \Psi - \Psi_k \Vert_{C(\cK;L^2(D))} + \Vert \cG - \cG_\ell \Vert_{C(\cK;L^2(D))} \right) 
\le 
\frac{18\epsilon}{s},
\]
and similarly for the last term.
We now find
\begin{align*}
\frac{\cL(\Psi;\cG) - \hL(\Psi;\cG)}{\cL(\Psi;\cG)+s}
&\le 
\frac{\big( \cL(\Psi_k;\cG_\ell) - \hL(\Psi_k;\cG_\ell) \big)_+}{\cL(\Psi_k;\cG_\ell)+s}
+ \frac{36 \epsilon}{s}.
\end{align*}
We can now bound the first term on the right by the maximum over all possible choices $k\in [N]$ and $\ell\in [M]$, and then take the supremum over all $(\Psi,\cG)\in \cH := \Sigma_m' \times \cU(\cA^\gamma)$ on the left, to find
\begin{align*}
\sup_{(\Psi,\cG)\in \cH} 
\frac{\cL(\Psi;\cG) - \hL(\Psi;\cG)}{\cL(\Psi;\cG)+s}
&\le 
\max_{k\in [N],\ell \in [M]}
\frac{\big( \cL(\Psi_k;\cG_\ell) - \hL(\Psi_k;\cG_\ell) \big)_+}{\cL(\Psi_k;\cG_\ell)+s}
+ \frac{36 \epsilon}{s}.
\end{align*}
Let us now make the choice $s = 72\epsilon/\alpha$, so that 
the above bound implies the following inclusion of sets:
\[
\left\{
\sup_{(\Psi,\cG)\in \cH} 
\frac{\cL(\Psi;\cG) - \hL(\Psi;\cG)}{\cL(\Psi;\cG)+s}
\ge \alpha
\right\}
\subset
\left\{
\max_{k\in [N],\ell \in [M]} 
\frac{\cL(\Psi;\cG) - \hL(\Psi;\cG)}{\cL(\Psi;\cG)+s}
\ge \frac{\alpha}{2}
\right\}.
\]
And hence,
\begin{align*}
\Prob\Bigg[
\sup_{(\Psi,\cG)\in \cH}
&\frac{\cL(\Psi;\cG) - \hL(\Psi;\cG)}{\cL(\Psi;\cG)+s}
\ge \alpha
\Bigg]
\\
&\le
\Prob\left[
\max_{k\in [N],\ell \in [M]}
\frac{\cL(\Psi_k;\cG_\ell) - \hL(\Psi_k;\cG_\ell)}{\cL(\Psi_k;\cG_\ell)+s}
\ge \frac{\alpha}{2}
\right]
\\
&\le NM \max_{k\in [N],\ell\in [M]} \Prob\left[
\frac{\cL(\Psi_k;\cG_\ell) - \hL(\Psi_k;\cG_\ell)}{\cL(\Psi_k;\cG_\ell)+s} \ge \frac{\alpha}{2}
\right].
\end{align*}
Where we use a union bound to pass to the third line. The last probability can be estimated by Bernstein's inequality \eqref{eq:bernstein}, yielding
\begin{align*}
\Prob\Bigg[
\sup_{(\Psi,\cG)\in \cH} 
\frac{\cL(\Psi;\cG) - \hL(\Psi;\cG)}{\cL(\Psi;\cG)+s}
\ge \alpha
\Bigg]
&\le
NM \exp\left( -\frac{\alpha^2 n s}{36} \right)
\\
&= 
NM \exp\left( -2\alpha n \epsilon \right).
\end{align*}
We also recall that $N = \cN(\Sigma_m',\epsilon)$ and $M = \cN(\cU(\cA^\gamma),2\epsilon)$. Since $\Sigma_m'\subset \Sigma_m$ implies $\cN(\Sigma_m',\epsilon) \le \cN(\Sigma_m,\epsilon)$, the above estimate proves the claimed inequality of Lemma \ref{lem:estprob} with constant 
\[
N_{\cH,\epsilon} = \cN(\Sigma_m,\epsilon) \cN(\cU(\cA^\gamma),2\epsilon),
\]
for arbitrary $\epsilon > 0$ and $m\in \N$. If we assume, in addition, that $m \ge (2/\epsilon)^{1/\gamma}$, then Lemma \ref{lem:coveringW1} implies that 
\[
\cN(\cU(\cA^\gamma),2\epsilon) \le \cN(\Sigma_m,\epsilon).
\]
And hence, we have 
\[
\Prob\Bigg[
\sup_{(\Psi,\cG)\in \cH} 
\frac{\cL(\Psi;\cG) - \hL(\Psi;\cG)}{\cL(\Psi;\cG)+s}
\ge \alpha
\Bigg]
\le
N_{\cH,\epsilon} \exp\left( -2\alpha n\epsilon \right),
\]
with $N_{\cH,\epsilon} \le \cN(\Sigma_m,\epsilon)^2$, 
whenever $m\ge (2/\epsilon)^{1/\gamma}$, and where we recall that we chose $s = 72 \epsilon/\alpha$.

\end{proof}

We next prove Corollary \ref{cor:estprob}.

\begin{proof}{(Proof of Corollary \ref{cor:estprob})}
Given $m\in \N$, define $\epsilon = 2m^{-\gamma}$, so that by Lemma \ref{lem:estprob}, we have
\[
N_{\cH,\epsilon} \le \cN(\Sigma_m,\epsilon)^2.
\]
We now note that 
\[
\cN(\Sigma_m,\epsilon)^2 \exp\left(-\frac{\epsilon^2 n}{9^2}  \right)
=
 \exp\left(-\frac{\epsilon^2 n}{9^2} + \log \cN(\Sigma_m,\epsilon)^2 \right).
\]
The right hand side is $\le \delta$, if 
\[
\frac{\epsilon^2 n}{9^2} \ge \log\left(\cN(\Sigma_m,\epsilon)^2/\delta\right),
\]
or equivalently,
\[
n \ge 81 \epsilon^{-2} \log\left(\cN(\Sigma_m,\epsilon)^2/\delta\right).
\]
Thus, by Lemma \ref{lem:estprob}, we find that for such choice of $n$:
\begin{align*}
\Prob\left[
\sup_{(\Psi,\cG)\in \cH} \Big[
\cL(\Psi;\cG) - \hL(\Psi;\cG)
\Big]
\ge 37\epsilon
\right]
&\le N_{\cH,\epsilon} \exp\left(-\frac{\epsilon^2 n}{9^2}  \right)
\le \delta.
\end{align*}
\end{proof}

\subsection{Entropy estimates}
\label{app:entropy}
Proposition \ref{prop:erm} relates the amount of data that is needed to approximate operators $\cG \in \cU(\cA^\gamma)$ to the complexity of the family $\{\Sigma_m\}$, as measured by the growth of the metric entropy $\log \cN(\Sigma_m,\epsilon)$.

\begin{remark}
\label{rem:algebraic}
If $\cH(\Sigma_m,\epsilon) \lesssim \epsilon^{-\alpha}$ grows at most algebraically, and $n \sim \epsilon^{-1-\alpha}$, then Proposition \ref{prop:erm} leads to a data-complexity bound,
\[
\E_{u\sim \mu}\left[ \Vert \Psi_\cG(u) - \cG(u) \Vert^2 \right]
\lesssim
n^{-1/(2+\alpha)},
\]
where the exponent $1/(2+\alpha)$ lies in the interval $(0,1/2]$. In particular, this shows that \emph{algebraic data-complexity rates} are obtained in this case. This should be contrasted with the logarithmic lower bound \eqref{eq:ol-gaussian1}, requiring an exponential amount of data $n$ to achieve a similar accuracy.
\end{remark}

\subsubsection{Proof of Lemma \ref{lem:lip-covering}}
\label{app:entropy1}

\begin{proof}{(Proof of Lemma \ref{lem:lip-covering})}
Fix $\epsilon > 0$. Let $\delta = \epsilon / \Lip(F)$, and let $B_\delta(u_1),\dots, B_\delta(u_N)$ be a $\delta$-covering of $\cK$. We claim that, 
\[
F(\cK) \subset \bigcup_{k=1}^N B_\epsilon(F(u_k)).
\]
To see why, let $v\in F(\cK)$ be given. We aim to show that $v \in B_\epsilon(F(u_k))$ for some $k\in \{1,\dots, N\}$. Since $v \in F(\cK)$, there exists $u\in \cK$, such that $v = F(u)$. By the choice of $u_1,\dots, u_N$, there exists an index $k$, such that $\Vert u - u_k \Vert < \delta$, and hence 
\[
\Vert v - F(u_k) \Vert 
=\Vert F(u) - F(u_k) \Vert \le \Lip(F) \Vert u - u_k \Vert < \epsilon,
\]
by definition of $\delta = \epsilon / \Lip(F)$. It follows that $v \in B_\epsilon(F(u_k))$. Choosing $N$ minimal, it thus follows that $\cN(F(\cK),\epsilon) \le N = \cN(\cK,\delta) = \cN(\cK,\epsilon/\Lip(F))$, as claimed.
\end{proof}

\subsubsection{Bounding the Lipschitz constant of $\theta \mapsto \Psi(\slot;\theta)$}
\label{app:fno-Lip}

Our next goal in this section is to estimate the Lipschitz constant of the mapping
\begin{align}
\label{eq:param-map}
[-B,B]^{d_\theta} \mapsto C(\cK;L^2(D)), 
\quad
\theta \mapsto \Psi(\slot;\theta),
\end{align}
for a fixed FNO architecture $\Psi(\slot;\theta)$ with parameters $\theta\in \R^{d_\theta}$, and with a fixed set of hyperparameters. We recall that the channel width is denoted $\dc$, the Fourier cut-off by $\kappa$, the depth $L$ and we have also introduced a parameter bound $\Vert \theta \Vert_{\ell^\infty} \le B$.

To reach our goal of estimating the Lipschitz constant of \eqref{eq:param-map}, we start by deriving basic estimates for the linear parts of the hidden layers. This is the subject of the following lemma:
\begin{lemma}
\label{lem:fno-hidden}
We have the following estimates:
\begin{enumerate}[(1)]
\item
If $W: v(x) \mapsto W v(x)$ is a linear layer, defined by a matrix $W\in \R^{\dc\times \dc}$, then
\begin{align}
\label{eq:estW}
\Lip(W)_{L^2\to L^2} \le \dc \Vert W \Vert_\infty,
\end{align}
where $\Vert W \Vert_\infty := \max_{i,j} |W_{ij}|$.
\item 
If $K: v(x) \mapsto \cF^{-1}(\hat{P} \cF v)$ is a (linear) spectral convolution layer, defined by the Fourier multiplier $\hat{P} \in \C^{\kappa \times \dc \times \dc}$, then
\begin{align}
\label{eq:estK}
\Lip(K)_{L^2\to L^2} \le \dc \Vert \hat{P} \Vert_\infty,
\end{align}
where $\Vert \hat{P} \Vert_\infty := \max_{k,i,j} |\hat{P}_{k,i,j}|$.
\item 
If $\sigma(x)$ is a Lipschitz continuous activation function, then 
\begin{align}
\label{eq:estSigma}
\Lip(\sigma)_{L^2\to L^2} \le \Lip(\sigma)_{\R \to \R}.
\end{align}
\item 
If $b(x) = \sum_{|k|< \kappa} \hat{b}_k e^{ikx}$ is a bias function, then 
\begin{align}
\label{eq:estb}
\Vert b \Vert_{L^2} \le \dc^{1/2} (2\kappa)^{d/2} \Vert \hat{b} \Vert_\infty,
\end{align}
where $\Vert \hat{b}\Vert_\infty = \max_k |\hat{b}_k|$.
\end{enumerate}
\end{lemma}

\begin{proof}
The claimed estimates are elementary. To see \eqref{eq:estW}, we note that
\[
\Vert Wv \Vert_{\ell^2}^2 
=
\sum_{i=1}^d \left| \sum_{j=1}^d W_{ij} v_j \right|^2
\le
\sum_{i=1}^d \sum_{j=1}^d |W_{ij}|^2  \sum_{j=1}^d |v_j|^2
\le
d^2 \Vert W\Vert_\infty^2 \Vert v \Vert_{\ell^2}^2.
\]
Similarly, \eqref{eq:estK} follows from,
\begin{align*}
\Vert K v \Vert_{L^2}^2
&=
\Vert \hat{P} \hat{v} \Vert_{\ell^2}^2
= 
\sum_{|k| < \kappa} \Vert \hat{P}_k \hat{v}(k) \Vert_{\ell^2}^2
\\
&\le
\sum_{|k| < \kappa} \dc^2 \Vert \hat{P}_k \Vert_\infty^2 \Vert \hat{v}(k) \Vert_{\ell^2}^2
\\
&\le \dc^2 \Vert \hat{P} \Vert_\infty^2 \sum_{|k| < \kappa} \Vert \hat{v}(k)\Vert_{\ell^2}^2
\\
&= 
\dc^2 \Vert \hat{P} \Vert_\infty^2 \Vert v \Vert_{L^2}^2.
\end{align*}
Estimate \eqref{eq:estSigma} is immediate, while \eqref{eq:estb} follows from,
\begin{align*}
\Vert b \Vert_{L^2}^2 
&= \sum_{|k| < \kappa} \Vert \hat{b}(k) \Vert_{\ell^2}^2 
\le \sum_{|k| < \kappa} \dc \Vert \hat{b}(k) \Vert_\infty^2
\le (2\kappa)^d \dc \Vert \hat{b} \Vert_\infty^2.
\end{align*}
\end{proof}

The next lemma, immediate from Lemma \ref{lem:fno-hidden}, summarizes the relevant Lipschitz estimate for a hidden layer:
\begin{lemma}
\label{lem:fno-hidden-lip}
Let $\cL_\ell(v; \theta_\ell) = \sigma\left(Wv + Kv + b\right)$ be a hidden FNO layer, with concatenated parameters $\theta_\ell = (W,\hat{P},\hat{b})$, satisfying the size bound
\[
\Vert W \Vert_\infty, \; \Vert \hat{P}\Vert_\infty, \; \Vert \hat{b} \Vert_{\infty} \le B,
\]
and with Lipschitz continuous activation function $\sigma$, s.t. $\Lip(\sigma)_{\R \to \R} \le 1$. Then $\cL_\ell: L^2(D) \to L^2(D)$ is a Lipschitz operator with, 
\begin{align}
\label{eq:fno-hidden-lip}
\Lip(v \mapsto \cL_\ell(v;\theta_\ell)) \le 2\dc B.
\end{align}
\end{lemma}

The previous corollary estimates the Lipschitz constant of the mapping $v\mapsto \cL_\ell(v;\theta_\ell)$ for fixed $\theta$ satisfying the bound $\Vert \theta_\ell \Vert_{\ell^\infty} \le B$. The next Lemma provides a similar estimate for the mapping $\theta_\ell \mapsto \cL_\ell(v;\theta_\ell)$ for fixed $v$.
\begin{lemma}
\label{lem:fno-layer-lip}
Let $\cL_\ell(v;\theta_\ell) = \sigma \left(Wv + Kv + b\right)$ denote a hidden FNO layer, with $1$-Lipschitz continuous activation $\sigma$, and depending on the concatenated parameters $\theta_\ell = (W,\hat{P},\hat{b})$. Then, for any $v \in L^2(D)$, we have the bound,
\[
\Lip(\theta_\ell \mapsto \cL(v;\theta_\ell)) 
\le
3\dc (2\kappa)^{d/2} \max(1,\Vert v \Vert_{L^2}).
\]
\end{lemma}

\begin{proof}
Let $v\in L^2(D)$ be arbitrary, consider two parameter vectors $\theta_\ell, \theta_\ell'$, and let us denote
\[
\cL_\ell(v;\theta_\ell) = \sigma( Wv + Kv + b), 
\quad
\cL_\ell(v;\theta_\ell') = \sigma( W'v + K'v + b').
\]
We note that, by the estimates in \ref{lem:fno-hidden}:
\begin{align*}
\Vert \cL_\ell(v;\theta_\ell) - \cL_\ell(v;\theta_\ell') \Vert
&\le \Lip(\sigma) ( 
\Vert W - W' \Vert_{L^2 \to L^2} \Vert v\Vert_{L^2} \\
&\qquad\qquad + \Vert K - K' \Vert_{L^2\to L^2} \Vert v \Vert_{L^2} + \Vert b - b' \Vert_{L^2} )
\\
&\le \dc \Vert W-W' \Vert_\infty \Vert v \Vert_{L^2} + \dc \Vert \hat{P} - \hat{P}' \Vert_\infty \Vert v \Vert_{L^2} \\
&\qquad + \dc^{1/2} (2\kappa)^{d/2} \Vert \hat{b} - \hat{b}' \Vert_{\infty}
\\
&\le 
\dc (2\kappa)^{d/2} ( 3 \max(1,\Vert v \Vert_{L^2})) \Vert \theta - \theta' \Vert_\infty.
\end{align*}
\end{proof}

In addition to the Lipschitz constant of all layers, we also need a bound on the output of hidden states of the FNO. This is provided in the next lemma:
\begin{lemma}
\label{lem:fno-layer-bd}
$\Psi$ FNO of depth $L$, with $1$-Lipschitz activation $\sigma$, such that $\sigma(0) = 0$. Assume that $\Vert \theta \Vert_{\ell^\infty} \le B$ for some constant $B \ge 1$.  Then 
\[
\Vert \Psi(u) \Vert_{L^2(D)} 
\le 
(2\dc B)^{L+2} \left( \Vert u \Vert_{L^2(D)} +(2\kappa)^{d/2} \right),
\quad
\forall \, u \in L^2(D).
\]
\end{lemma}

\begin{proof}
Recall that the FNO $\Psi: L^2(D) \to L^2(D)$ is of the form 
\[
\Psi = Q \circ \cL_L \circ \dots \circ \cL_1 \circ P.
\]
For $\ell = 1,\dots, L+1$, we introduce the notation
\begin{align*}
\cL_{\le \ell} &:= \cL_{\ell} \circ \dots \circ \cL_1 \circ P,
\end{align*}
and define $\cL_{\le 0}(u) := Pu$. We will estimate $\Vert \cL_{\le \ell}(u) \Vert$ recursively, making use of our bound on $\Lip(v \mapsto \cL_\ell(v))$ in Lemma \ref{lem:fno-hidden-lip}. To this end, fix $u\in L^2(D)$, and denote $M_\ell := \Vert \cL_{\le \ell}(u) \Vert_{L^2(D)}$. Then, writing $\cL_{\ell}(v) := \sigma(Wv+Kv+b)$, we obtain
\begin{align*}
M_\ell &= \Vert \cL_{\le \ell}(u) \Vert_{L^2(D)}
\\
&= \Vert \cL_{\ell} \left( \cL_{\le \ell-1}(u) \right) \Vert_{L^2(D)}
\\
&= \Vert \cL_{\ell} \left( \cL_{\le \ell-1}(u) \right) - \cL_{\ell}(0) \Vert_{L^2(D)} + \Vert \cL_\ell(0) \Vert_{L^2(D)}
\\
&\le 
\Lip(\cL_{\ell}) \Vert \cL_{\le \ell-1}(u) - 0 \Vert_{L^2(D)} 
+ \Vert \sigma(b(\slot)) \Vert_{L^2(D)}
\\
&\le 
\Lip(\cL_{\ell}) \Vert \cL_{\le \ell-1}(u)\Vert_{L^2(D)} 
+ \Lip(\sigma) \Vert b(\slot) \Vert_{L^2(D)}
\\
&\le
(2\dc B) \Vert \cL_{\le \ell-1}(u) \Vert_{L^2(D)} + \dc^{1/2} (2\kappa)^{d/2} B,
\end{align*}
where we made use of \eqref{eq:fno-hidden-lip} and \eqref{eq:estb} to pass to the last line. Thus, we have the following recursion:
\[
M_\ell \le C_0 M_{\ell-1} + C_1,
\]
with $C_0 := 2\dc B$ and $C_1 := \dc^{1/2} (2\kappa)^{d/2} B$. This recursion implies that 
\begin{align*}
M_{\ell} 
&\le C_0 M_{\ell-1} + C_1 \\
&\le C_0^2 M_{\ell-2} + C_0 C_1 + C_1\\
&\le \dots \\
&\le C_0^\ell M_0 + C_1 \sum_{k=0}^{\ell-1} C_0^k \\
&= C_0^\ell M_0 + C_1 \frac{C_0^\ell - 1}{C_0 - 1}.
\end{align*}
For $C_0 = 2\dc B \ge 2$, we have 
\[
\frac{C_0^\ell - 1}{C_0 - 1} \le C_0^{\ell}.
\]
Furthermore, by definition of $M_0$ and from \eqref{eq:estW}, we find
\[
M_0 = \Vert Pu\Vert_{L^2(D)} 
\le \dc \Vert P \Vert_\infty \Vert u \Vert_{L^2(D)}
\le \dc B \Vert u \Vert_{L^2(D)}.
\]
Thus, we finally obtain the upper bound,
\begin{align}
\label{eq:Lless}
\begin{aligned}
\Vert \cL_{\le \ell}(u) \Vert_{L^2(D)} 
&= M_\ell 
\\
&\le (2\dc B)^\ell \left( \dc B \Vert u \Vert_{L^2(D)} + \dc^{1/2} (2\kappa)^{d/2} B \right)
\\
&\le (2\dc B)^{\ell+1} \left( \Vert u \Vert_{L^2(D)} + (2\kappa)^{d/2} \right).
\end{aligned}
\end{align}
This, in turn, entails
\begin{align*}
\Vert \Psi(u) \Vert_{L^2(D)}
&=
\Vert Q \cL_{\le L}(u) \Vert_{L^2(D)}
\le 
\dc \Vert Q \Vert_\infty M_L
\\
&\le 
(2\dc B)^{L+2} \left( \Vert u \Vert_{L^2(D)} +(2\kappa)^{d/2} \right).
\end{align*}
This is the claimed upper bound.
\end{proof}

Based on the previous lemmas, we can now state a local Lipschitz estimate for FNOs. 
\begin{proposition}
\label{prop:fno-Lip}
Let $\theta, \theta'\in \R^{d_\theta}$ with $\Vert \theta\Vert_\infty, \Vert \theta'\Vert_\infty\le B$. Assume that $B\ge 1$ and $d_{\mathrm{in}}, d_{\mathrm{out}} \le d_c$. Assume that the activation $\sigma$ is Lipschitz and $\sigma(0) = 0$. Then, we have
\[
\Vert \Psi(u;\theta) - \Psi(u,\theta') \Vert_{L^2}
\le
\lambda(u) \Vert \theta - \theta' \Vert_\infty,
\]
where 
\[
\lambda(u) := (L+2) (2\dc B)^{L+2} (\Vert u \Vert + (2\kappa)^{d/2}).
\]
\end{proposition}

Lemma \ref{lem:fno-Lip} is an immediate consequence of Proposition \ref{prop:fno-Lip}.

\begin{proof}{(Proof of Proposition \ref{prop:fno-Lip})}
We first consider two parameter vectors $\theta, \theta'\in \R^{d_\theta}$, differing in only their $\ell$-th component for $\ell \in \{1,\dots, L\}$, and with 
\[
\Vert \theta \Vert_\infty , \Vert \theta' \Vert_\infty \le B.
\]
Notice that, by definition,
\[
\Psi(\slot;\theta) = Q \circ \cL_L \circ \dots \circ \cL_1 \circ P,
\]
where the $\ell$-th hidden layer $\cL_\ell = \cL_\ell(\slot; \theta_\ell)$ depends only on $\theta_\ell$. To simplify notation, let us now introduce
\begin{align*}
\cL_{>\ell} &:= \cL_L \circ \dots \cL_{\ell+1}, 
\\
\cL_{<\ell} &:= \cL_{\ell-1} \circ \dots \circ \cL_1 \circ P.
\end{align*}
We then have,
\begin{align*}
\Vert \Psi(u;\theta) - \Psi(u,\theta') \Vert_{L^2}
&\le \Vert Q \Vert \Lip(\cL_{>\ell}) \Vert \cL_\ell(\cL_{<\ell}(u);\theta_\ell) - \cL_\ell(\cL_{<\ell}(u);\theta_\ell') \Vert.
\end{align*}
We note that $\Vert Q \Vert \le \dc \Vert Q \Vert_\infty \le \dc B$, by Lemma \ref{lem:fno-hidden}. By Lemma \ref{lem:fno-hidden-lip}, we can bound
\begin{align*}
 \Lip(\cL_{>\ell})
 &\le \prod_{k=\ell+1}^L \Lip(\cL_k)
 \le  (2\dc B)^{L-\ell}.
\end{align*}
 Furthermore, Lemma \ref{lem:fno-layer-lip} implies that, with $\omega:= \cL_{<\ell}(u)$,
\[
\Vert \cL_\ell(\omega;\theta_\ell) - \cL_\ell(\omega;\theta_\ell') \Vert
\le 
3\dc (2\kappa)^{d/2}  \max(1,\Vert \omega \Vert_{L^2})) \Vert \theta - \theta' \Vert_\infty.
\]

Using the following upper bound, from \eqref{eq:Lless} in the proof of Lemma \ref{lem:fno-layer-bd},
\[
\Vert \omega \Vert_{L^2(D)}
=
\Vert \cL_{<\ell}(u) \Vert_{L^2(D)}
=
\Vert \cL_{\le \ell-1}(u) \Vert_{L^2(D)}
\le
 (2\dc B)^{\ell+1} ( \Vert u \Vert + (2\kappa)^{d/2} ),
\]
we thus find
\begin{align*}
\Vert \cL_\ell(\omega;\theta_\ell) - \cL_\ell(\omega;\theta_\ell') \Vert
&\le  
3\dc (2\kappa)^{d/2}  (2\dc B)^{\ell+1} ( \Vert u \Vert + (2\kappa)^{d/2} ) \Vert \theta - \theta' \Vert_\infty.
\end{align*}
And finally,
\begin{align*}
\Vert \Psi(u;\theta) - \Psi(u,\theta') \Vert_{L^2}
&\le \Vert Q \Vert \Lip(\cL_{>\ell}) \Vert \cL_\ell(\cL_{<\ell}(u);\theta_\ell) - \cL_\ell(\cL_{<\ell}(u);\theta_\ell') \Vert
\\
&\le \dc B (2\dc B)^{L-\ell} (2\dc B)^{\ell+1} (\Vert u \Vert + (2\kappa)^{d/2}) \Vert \theta - \theta' \Vert_\infty
\\
&\le (2\dc B)^{L+2} (\Vert u \Vert + (2\kappa)^{d/2}) \Vert \theta - \theta' \Vert_\infty.
\end{align*}

It follows that for general $\theta, \theta'\in \R^{d_\theta}$ with $\Vert \theta\Vert_\infty, \Vert \theta'\Vert_\infty\le B$, but allowed to differ in $\theta_\ell \ne \theta_\ell'$ for $\ell = 0, \dots, L+1$, we have
\[
\Vert \Psi(u;\theta) - \Psi(u,\theta') \Vert_{L^2}
\le
\lambda(u) \Vert \theta - \theta' \Vert_\infty,
\]
where 
\[
\lambda(u) := (L+2) (2\dc B)^{L+2} (\Vert u \Vert + (2\kappa)^{d/2}).
\]
This provides the claimed upper bound on the local Lipschitz constant of the mapping $\R^{d_\theta} \to C(L^2(D);L^2(D))$, $\theta \mapsto \Psi(\slot;\theta)$. 
\end{proof}

\end{document}

%% file: custom_defs.tex
\newcommand{\U}{\mathsf{U}}

\newcommand{\size}{\mathrm{size}}

\newcommand{\dc}{{d_c}}

\newcommand{\FNO}{\bm{\mathsf{FNO}}}

\renewcommand{\tilde}{\widetilde}
\renewcommand{\hat}{\widehat}

\DeclareMathOperator*{\argmin}{argmin}

\newcommand{\embeds}{{\hookrightarrow}}

\renewcommand{\bar}{\overline}

\newcommand{\slot}{{\,\cdot\,}}

\newcommand{\supp}{\mathrm{supp}}

\newcommand{\R}{\mathbb{R}}
\newcommand{\C}{\mathbb{C}}
\newcommand{\E}{\mathbb{E}}
\newcommand{\N}{\mathbb{N}}

\newcommand{\Z}{\mathbb{Z}}
\newcommand{\Prob}{\mathrm{Prob}}

\newcommand{\Lip}{\mathrm{Lip}}

\newcommand{\cA}{\mathsf{\bm{A}}}
\newcommand{\cB}{\mathcal{B}}

\newcommand{\cD}{\mathcal{D}}
\newcommand{\cE}{\mathcal{E}}
\newcommand{\cF}{\mathcal{F}}
\newcommand{\cG}{\mathcal{G}}
\newcommand{\cH}{\mathcal{H}}

\newcommand{\cK}{\mathcal{K}}
\newcommand{\cL}{\mathcal{L}}

\newcommand{\cN}{\mathcal{N}}

\newcommand{\cP}{\mathcal{P}}

\newcommand{\cU}{\mathcal{U}}
\newcommand{\cV}{\mathcal{V}}
\newcommand{\cW}{\mathcal{W}}
\newcommand{\cX}{\mathcal{X}}
\newcommand{\cY}{\mathcal{Y}}
\newcommand{\cZ}{\mathcal{Z}}

\newcommand{\hL}{\widehat{\cL}}

\renewcommand{\hat}{\widehat}

\newcommand{\set}[2]{{\left\{ #1 \,\middle|\, #2 \right\}}}

\newcommand{\Enc}{\mathcal{E}}
\newcommand{\D}{\mathcal{D}}

%% file: custom_thms.tex
\declaretheoremstyle[
  headfont=\normalfont\bfseries\itshape,
  numbered=unless unique,
  bodyfont=\normalfont,
  spaceabove=1em plus 0.75em minus 0.25em,
  spacebelow=1em plus 0.75em minus 0.25em,
  qed={},
]{deflt}

\theoremstyle{deflt}
\newtheorem{theorem}{Theorem}[section]

\newtheorem{remark}[theorem]{Remark}

\newtheorem{lemma}[theorem]{Lemma}
\newtheorem{proposition}[theorem]{Proposition}
\newtheorem{corollary}[theorem]{Corollary}

\numberwithin{equation}{section}
\numberwithin{theorem}{section}
